\documentclass[11pt,a4paper]{article}

\usepackage[utf8]{inputenc}
\usepackage[T1]{fontenc}
\usepackage{microtype}
\usepackage{graphicx}
\usepackage{subcaption}
\usepackage{booktabs} 
\usepackage[margin=1in]{geometry}

\usepackage{hyperref}
\hypersetup{
  colorlinks=true,
  linkcolor=blue,
  citecolor=blue,
  urlcolor=blue
}

\usepackage{amsmath}
\usepackage{amssymb}
\usepackage{mathtools}
\usepackage{amsthm}

\usepackage{bbm}
\def\mb{\ensuremath\mathbb}
\def\mc{\ensuremath\mathcal}

\DeclareMathOperator{\Cov}{\textnormal{Cov}}

\usepackage{algorithm}
\usepackage{algorithmic}

\usepackage[capitalize,noabbrev]{cleveref}

\newcommand{\Idd}{\textnormal{I}_d}

\theoremstyle{plain}
\newtheorem{theorem}{Theorem}[section]

\newtheorem{lemma}[theorem]{Lemma}
\newtheorem{corollary}[theorem]{Corollary}
\theoremstyle{definition}

\newtheorem{assumption}[theorem]{Assumption}
\theoremstyle{remark}
\newtheorem{remark}[theorem]{Remark}

\usepackage[textsize=tiny]{todonotes}

\title{Total Variation Guarantees for Sampling with Stochastic Localization}

\author{
  Jakob Kellermann\\[6pt]
  \textit{Weierstrass Institute Berlin} \\
  [4pt]
  Correspondence: \texttt{jakob.kellermann@wias-berlin.de}
}

\date{March 31, 2026}

\begin{document}

\maketitle

\begin{abstract}
  Motivated by the success of score-based generative models, a number of diffusion-based algorithms have recently been proposed for the problem of sampling from a probability measure whose unnormalized density can be accessed. Among them, Grenioux et al.\ introduced SLIPS \cite{SLIPS}, a sampling algorithm based on Stochastic Localization. While SLIPS exhibits strong empirical performance, no rigorous convergence analysis has previously been provided. In this work, we close this gap by establishing the first guarantee for SLIPS in total variation distance. Under minimal assumptions on the target, our bound implies that the number of steps required to achieve an $\varepsilon$-guarantee scales linearly with the dimension, up to logarithmic factors.
  The analysis leverages techniques from the theory of score-based generative models and further provides theoretical insights into the empirically observed optimal choice of discretization points.
\end{abstract}

\section{Introduction}
In this work, we consider the problem of sampling from a probability measure $\pi$ whose unnormalized density is accessible, namely $\pi(x)\propto e^{-f(x)}$ for some $f:\mb{R}^d\to \mb{R}$ that we can evaluate. 
This problem is of interest to many scientific domains, among them Bayesian statistics \cite{robert2004monte}, \cite{kroese2011handbook}, statistical physics \cite{krauth2006statistical}, molecular dynamics \cite{stoltz2010free}, and generative modeling \cite{turner2019metropolis}, \cite{grenioux2023balanced}. 
In this work, we establish theoretical guarantees for a novel sampling algorithm, \textit{Stochastic Localization via Iterative Posterior Sampling} (SLIPS), which was introduced recently by Grenioux, Noble, Durmus, and Gabri\'e in \cite{SLIPS}.
\vspace{2mm}\\A number of algorithmic approaches have been developed for the sampling problem, among them prominently the Markov chain Monte Carlo (MCMC) method. We recall that in MCMC, one aims to sample from $\pi$ by simulating a Markov chain whose invariant measure is $\pi$, at least approximately, for long enough time. The approximate sample is then given by the final iterate of the chain. While these approaches provably perform well for distributions fulfilling certain assumptions, for example log-concavity \cite{chewi2024logconcave}, they have poor performance for a large number of practically relevant classes of distributions. One highly relevant example is the class of probability measures whose density functions have several local maximizers that are well separated in space. One refers to the maximizers as \textit{modes} and, hence, to the distributions as \textit{multi-modal}. Notably, this class includes Gaussian mixture models with well separated means \cite{mixture_models_book}. Local MCMC methods fail to sample from such distributions due to the phenomenon of mode collapse, e.g.\ within polynomial time the chain does not manage to explore the regions surrounding all modes but focuses only on a single one \cite{grenioux2025improving}.
In the past, annealing approaches have been developed that aim to tackle the problem of sampling from a target by reducing it to a collection of simpler sampling problems. These can then be solved either sequentially or in parallel, for example using MCMC \cite{replica_MC}, \cite{Geyer1991MCMCML}, \cite{sequential_MC}. Such algorithms can in principle sample from multi-modal distributions. 
Though, in high dimensions and in the presence of sufficiently strong mode separation, also algorithms of this type fail \cite{grenioux2025improving}.
\vspace{2mm}\\An alternative to Monte Carlo based sampling is Variational Inference (VI) \cite{book_variational_inference}. In VI, one considers a parametrized family of distributions $(\pi_{\theta})_{\theta\in I}$ from which one can sample efficiently. 
Here, $I$ is some index set, for example $I=\mb{R}^d$, and $(\pi_{\theta})_{\theta\in I}$ is commonly referred to as the variational family.
The goal is then to identify a parameter $\theta\in I$ for which $\pi_\theta\approx \pi$. To do so, one optimizes a variational objective between $\pi_{\theta}$ and $\pi$, often the KL divergence. 
Hence, one turns the sampling problem into an optimization problem. The main computational effort lies in learning $\theta$ with $\pi_\theta\approx \pi$. Once $\theta$ is learned, sampling comes with minimal computational effort, since we assume that $\pi_\theta$ is easy to sample. One refers to such methods as \textit{amortized sampling} approaches.
\vspace{2mm}\\There is a number of VI approaches that perform well in high dimensions, for example VI via normalizing flows \cite{normalizing_flows_var_inf} or variational auto encoders \cite{var_autoencoders}, \cite{kingma2014auto}.
Though, in the case of multi-modal targets, VI methods are known to suffer often from mode collapse, even when the variational family is able to capture multi-modal distributions. This is due to the KL optimization objective guiding towards uni-modal distributions \cite{var_inf_paper_1}, \cite{var_inf_paper_2}. Recently, Gabri\'e et al.\ provided some theoretical understanding of this phenomenon \cite{var_inf_mode_collapse}.
As for Monte Carlo sampling, there are approaches relying on annealing strategies for VI, see \cite{annealed_VI_1_wu2020stochastic}, \cite{annealed_VI_2_arbel2021annealed} and \cite{annealed_VI_3_doucet2022score}. These approaches improve on the performance of classical VI methods, but also struggle with substantial multi-modality and high dimensions, as annealed Monte Carlo samplers do.
\vspace{2mm}\\Different to the setting where one can access the unnormalized density function but has no samples a priori, in practice one often faces the opposite challenge: There is lots of data, theoretically modeled as samples from a probability measure, but one has no access to the density. The problem of generating new samples given the data is referred to as \textit{generative modeling}.
Recently, state of the art performance in generative modeling has been achieved with diffusion based models \cite{sohl-dickstein15},
\cite{song_ermon}, \cite{denoising_diffusion},
\cite{GM_song2021scorebased}. These models are referred to as score-based generative models (SGMs), due to the score of the diffusion's marginal distribution appearing in the drift of its SDE description. This score term is intractable and estimating it via a neural network using the data at hand is the main challenge of score-based generative modeling.
Under the assumption of access to a good score estimator, recently there have also been established rigorous theoretical guarantees for SGMs, giving insight into the strong performance observed in practice \cite{sgm_lee}, \cite{sgm_lee_2}, \cite{sgm_chen}, \cite{Sampling_is_as_easy_as_learning_the_score_chen2023sampling}, \cite{sgm_li}, \cite{benton2024nearly}, \cite{ConfortiDurmusSilveri2025}.
\vspace{2mm}\\Motivated by the success of SGMs, several new diffusion based sampling approaches have been proposed for the problem of sampling from an unnormalized density. 
These can be separated broadly in two classes, that of diffusion samplers based on Monte Carlo \cite{SLIPS}, \cite{huang2024reverse}, \cite{he2024zerothorder} and that based on Variational Inference, see e.g.\ \cite{path_integral_sampler}, \cite{denoising_diffusion_samplers_vargas2023}, \cite{optimal_control_berner2024}, \cite{richter2024improved}. The essential difference is that in the former case the drift term of the diffusion's SDE is estimated in each iteration via Monte Carlo methods, whereas in the latter case it is learned using Variational Inference techniques involving a neural network.  The SLIPS algorithm belongs to the former class.
When properly initialized, such algorithms have powerful performance in numerical experiments, outperforming virtually all competitors 
\cite{noble_lrds},
\cite{grenioux2025improving}.
Though, different to classical methods such as MCMC, they currently lack a thorough mathematical convergence analysis. 

\paragraph{Contributions.}
We apply the techniques developed for the study of SGMs to the Stochastic Localization based SLIPS sampling algorithm proposed in \cite{SLIPS}, giving novel convergence guarantees in total variation (TV) distance with linear dimension dependence under weak assumptions on the target. Further, based on our TV estimates, we provide theoretical insights into the optimal empirical performance of the log-SNR adapted choice of discretization.

\paragraph{Outline.} In Section \ref{section_SLIPS}, we give an introduction to Stochastic Localization, present the SLIPS algorithm as introduced in \cite{SLIPS}, and explain how it relates to the default SGM. Section \ref{section_TV} contains our main result, a TV guarantee for SLIPS. We also provide an overview on the proof structure. The results on the optimality of the log-SNR adapted discretization in the TV guarantee for SLIPS are the content of Section \ref{section_discretization}. 
We compare our guarantee derived for SLIPS with a guarantee for a conceptually similar diffusion-based sampler introduced in \cite{huang2024reverse} in Section \ref{section_slips_rdmc}.
Finally, we draw conclusions in Section \ref{section_conclusion}. The Appendix contains the proofs of our results and additional supplementary material.

\paragraph{Notation.} We write $\mc{P}(\mb{R}^d)$ for the space of Borel probability measures on $\mb{R}^d$. For $p\geq 1$ we 
write $\mc{P}_p(\mb{R}^d)$ for the subset of measures fulfilling $\int_{\mb{R}^d}\,\|x\|^p\,\pi(dx)<\infty$. For $x,m\in \mb{R}^d$ and $\sigma^2>0$, we denote by $\mc{N}(x,m,\sigma^2)$ the density of $\mc{N}(m,\sigma^2)$ evaluated in $x$, where we write here and in the following by an abuse of notation $\mc{N}(m,\sigma^2)$ instead of $\mc{N}(m,\sigma^2\cdot \textnormal{I}_d)$.

\section{SLIPS sampling algorithm}\label{section_SLIPS}
We discuss the concept of Stochastic Localization and present the SLIPS algorithm introduced in \cite{SLIPS}. We also discuss its relation to generative modeling.

\subsection{Stochastic Localization}\label{section_sub_SL}
Stochastic Localization is a concept that has its origin in geometric measure theory \cite{Eldan2013}, \cite{Chen2021KLS}, \cite{eldan_pathwise}. It became relevant to probability first as a tool for proving mixing time estimates for Markov chains \cite{SL_chen_eldan} and later as an approach to sampling itself \cite{montanari_SK_model_2022}, \cite{montanari2023samplingdiffusionsstochasticlocalization}, \cite{montanari2024posteriorsamplinghighdimension}. The core idea of Stochastic Localization for sampling as outlined in the blueprint \cite{montanari2023samplingdiffusionsstochasticlocalization} is to construct a process $Y=(Y_{t})_{t\in [0,\infty)}$ that becomes increasingly informative about an underlying random variable $X\sim \pi$, with perfect information in the limit. Further, one requires the information to be, heuristically speaking, stable across time. 
The choice of the time interval $[0,\infty)$ is done here for simplicity.
\vspace{2mm}\\More precisely, one asks the pair $X,Y$ to fulfill the following two properties.
\begin{enumerate}
    \item The posterior measure process $\big(\mb{P}(X\in \cdot \,|Y_t)\big)_{t\in [0,\infty)}$ converges almost surely weakly to $\delta_{X\in \cdot}$ in the limit of $t\to \infty$.
    \item For all $A\in \mc{B}(\mb{R}^d)$ the process $\big(\mb{P}(X\in A|Y_t)\big)_{t\in [0,\infty)}$ is a martingale.
\end{enumerate}
The first requirement corresponds to perfect asymptotic information, whereas the second corresponds to stable information.
\vspace{2mm}\\While often one has additional to the first property that some rescaling of $Y$ converges towards $X$, in general this is not necessary and $Y$ and $X$ might not even take values in the same space, see Section 4 in \cite{montanari2023samplingdiffusionsstochasticlocalization}.

\subsection{SLIPS: The standard case}\label{section_sub_slips_standard}
Let $X\sim \pi$, $B$ an independent Brownian motion, and $\sigma>0$. In SLIPS, Grenioux et al.\ propose using the \textit{Stochastic Localization observation process} defined through
\begin{align}
    (Y_t)_{t\geq 0} \coloneqq (tX+\sigma B_t)_{t\geq 0}.
\end{align}
This choice is motivated by the fact that $Y$ localizes on $X$, e.g.\ almost surely $\lim_{t\to \infty}Y_t/t= X$. Further, the pair $X,Y$ fulfills the properties demanded in Section \ref{section_sub_SL}. Provided it is feasible to simulate $Y$ up to some sufficiently large time $T>0$, an approximate sample from $\pi$ can be obtained by means of $Y_T/T$.
\vspace{2mm}\\In the following, for $t\geq 0$ we write $p_t\coloneqq \textnormal{Law}(Y_t)$ and note that for $y\in \mb{R}^d$ we have
\begin{align}\label{eq_p}
    p_t(y) = \int_{\mb{R}^d} \mc{N}(y,tx,t\sigma^2)\,\pi(x)\,dx.
\end{align}
Hence, both computing $p_t(y)$ exactly as well as only up to the normalization constant is infeasible.
Let us define for $(x,y)\in \mb{R}^d\times \mb{R}^d$ and $t>0$ the family of posterior densities of $X=x$ given $Y=y$ by
\begin{align}\label{eq_def_posterior}
    q_t(x|y)\coloneqq  \frac{\mc{N}(y,tx,t\sigma^2) \pi(x)}{p_t(y)}=\mb{P}(X=x|Y_t=y).
\end{align}
The associated posterior expectation, also referred to as the \textit{optimal denoiser} is defined by
\begin{align}\label{eq_posterior_expectation}
    u_t(y)\coloneqq\int_{\mb{R}^d}x\,q_t(dx|y)
    =\mb{E}[X|Y_t=y].
\end{align}
Provided $\pi \in \mc{P}_1(\mb{R}^d)$, the SL observation process $Y$ solves the SDE
\begin{align}\label{eq_SDE_Y}
    dY_t=u_t(Y_t)\,dt+\sigma\,d\bar{B}_t,
\end{align}
where $\bar{B}$ is a Brownian motion different from $B$, see Corollary 11 in \cite{SLIPS}. Provided $\pi\in \mc{P}_2(\mb{R}^d)$, the process $Y$ is also the unique solution to \eqref{eq_SDE_Y}, see Corollary 13 in \cite{SLIPS}.
Simulation of $Y$ can then be achieved by discretizing \eqref{eq_SDE_Y} using Euler Maruyama (EM). Consider discretization points $0\leq t_0<t_1<\dots<t_K$. Setting $\delta_k=t_{k+1}-t_k$, recall that the EM recursion reads
\begin{align}\label{eq_EM_recursion}
    \widetilde{Y}_{t_{k+1}} & = \widetilde{Y}_{t_k}+ \delta_k \cdot u_{t_k}(\widetilde{Y}_{t_k}) + \sigma \sqrt{\delta_k}\cdot\mc{G}_{k},
\end{align}
where $(\mc{G}_{k})_{0\leq k\leq K-1}$ is a family of i.i.d.\ standard Gaussians taken independently of the initialization $\widetilde{Y}_{t_0}\sim p_{t_0}$. 
\vspace{2mm}\\The caveat now lies in the fact that to run the recursion \eqref{eq_EM_recursion}, one needs to compute the drift $u_{t_k}(\widetilde{Y}_{t_k})$ in each step. In our setup, we thus have to compute the expectation \eqref{eq_posterior_expectation}, which is not available in closed form. Hence, we have to approximate it. 
\vspace{2mm}\\Grenioux et al.\ propose to do so by generating in the $k$-th step of the discretization $M$ samples $\{X^k_j\}_{1\leq j \leq M}$ from the posterior measure $q_{t_k}(\cdot|\widetilde{Y}_{t_k})$ by simulating a single Markov chain. They propose using the Metropolis adjusted Langevin algorithm (MALA) \cite{roberts_tweedie_langevin}. Their estimate of $u_{t_k}(\widetilde{Y}_{t_k})$ is the empirical average of the samples,
\begin{align}
    \hat{u}_{t_k}(\widetilde{Y}_{t_k}) = \frac{1}{M} \sum_{1\leq j \leq M} X^k_j.
\end{align}
To obtain a good estimate, it is thus essential that sampling $q_{t_k}(\cdot|\widetilde{Y}_{t_k})$ via MCMC is feasible. 
As is evident from \eqref{eq_def_posterior}, for small $t$ the contribution of $\pi$ to the posterior dominates, while for large $t$ the Gaussian contribution does. In particular, since we are primarily interested in settings where sampling from $\pi$ directly is infeasible, sampling the posterior is possible only for $t$ large enough. Therefore, one has to choose $t_0> 0$ sufficiently large. 
\vspace{2mm}\\This gives rise to the second challenge, sampling the initialization $p_{t_0}$. Since $Y_{t_0}=t_0 X+\sigma B_{t_0}$, the Gaussian contribution dominates $p_{t_0}$ only for $t_0$ small enough. Hence, only then we can expect that sampling $p_{t_0}$ via MCMC methods is feasible.
\vspace{2mm}\\We would like to apply the Unadjusted Langevin Algorithm (ULA) \cite{roberts_tweedie_langevin} to sample from $p_{t_0}$. To do so, we have to compute in each step the score of the distribution, namely $\nabla \log p_{t_0}$. As can be seen from \eqref{eq_p}, this quantity is in general not accessible. Though, the score and optimal denoiser fulfill a relation referred to as \textit{Tweedie's formula}, which reads
\begin{align}\label{eq_tweedie}
    u_t(y)=y/t + \sigma^2 \,\nabla \log p_t(y)
\end{align}
and follows by a basic computation, see Lemma 4 in \cite{SLIPS}.
Grenioux et al.\ propose estimating the score in each iteration of ULA by estimating the optimal denoiser, as is done in the discretization of \eqref{eq_SDE_Y}, and to turn this in a score estimate by exploiting \eqref{eq_tweedie}. Since in each step of applying ULA one has to apply MALA, the authors coin the resulting algorithm Langevin-within-Langevin initialization, which is stated as Algorithm \ref{alg:langevin_within_langevin}.
\vspace{2mm}\\Observe that now there is a trade-off in choosing $t_0$: Since estimation of the drift involves sampling from the posterior, it requires taking $t_0$ large enough, while sampling the initialization requires taking $t_0$ small enough. Grenioux et al.\ coin this trade-off the \textit{duality of log-concavity}. This is due to the fact that the mentioned dominance of the Gaussian contributions leads to the involved measures being log-concave, ensuring that they are feasible to sample. They prove, under a restrictive condition on $\pi$, that $t_0$ can be chosen indeed in such a way that all involved measures are log-concave, see Theorem 3 in \cite{SLIPS}. While the condition is restrictive, Grenioux et al.\ demonstrate numerically that SLIPS works well beyond the theoretically covered setup. 
\vspace{2mm}\\Still, the numerical experiments also demonstrate that the performance of SLIPS has a $v$-shape dependence on the choice of $t_0$, i.e.\ there is a sweet spot for it. Notably, the dependence is strong, so for too large or too small $t_0$ the algorithm fails to sample from $\pi$. Further, in practice, the higher the dimension is, the smaller one has to choose $t_0$. Importantly, if $t_0$ is chosen appropriately, the SLIPS algorithm performs very well on highly multi-modal sampling problems, outperforming most competitors \cite{SLIPS}, \cite{grenioux2025improving}. 
\vspace{2mm}\\In practice, for the posterior expectation estimation in the SDE discretization, a quite small number of MALA iterations is sufficient ($\ll 100$). The reason for that is twofold. First, the posterior sampling problem becomes increasingly simple for later iterations due to the Gaussian becoming more dominant, see \eqref{eq_def_posterior}. Second, the initialization of the MALA chain for the $k+1$-st posterior expectation estimation problem is chosen as the last iterate of the MALA chain used in the previous step. Heuristically, as the posterior measures of the $k$-th and the $k+1$-st step should look very similar, this can be thought of as a warm start.  

\subsection{SLIPS: General denoising schedules}\label{section_sub_slips_general} 
The key property to utilize $Y_t$ for sampling is that a.s.\
$\lim_{t\to \infty} Y_t/t =X$.
Hence, it is natural to consider replacing the $t$ scaling of $X$ in $Y$ by a more general \textit{denoising schedule} $\alpha:[0,T_{\textnormal{gen}})\to \mb{R}_+$, where $T_{\textnormal{gen}}\in (0,\infty]$. While for the exact definition we refer to Section 3.1 in \cite{SLIPS}, the key properties that one asks for are that $\alpha$ is increasing, fulfills $\alpha(0)=0$, and is chosen such that almost surely
\begin{align}
    \lim_{t\to T_{\textnormal{gen}}} Y^\alpha_t/\alpha(t)=X,
\end{align}
where $(Y_t^\alpha)_{0\leq t <T_{\textnormal{gen}}} \coloneqq (\alpha(t) X + \sigma B_t)_{0\leq t <T_{\textnormal{gen}}}$. 
To sample using $Y^\alpha$, one wishes to identify its SDE description. Given the SDE \eqref{eq_SDE_Y} for the standard case and since trivially 
\begin{align}
    dY^\alpha_t = \alpha'(t)\, X\,dt+ \sigma \,dB_t,    
\end{align}
it is a priori natural to expect that, provided $\pi\in \mc{P}_1(\mb{R}^d)$, one has
\begin{align}\label{eq_SDE_general}
    dY^\alpha_t = \alpha'(t)\cdot u_t^\alpha (Y_t^\alpha)\,dt+ \sigma \,d\bar{B}_t,
\end{align}
where $u_t^\alpha (Y_t^\alpha)\coloneqq \mb{E}[X|Y_t^\alpha]$ and $\bar{B}$ is a BM different to $B$. It turns out that this is in fact \textit{not} the case. 
Even though the process $Y^\alpha$ does solve an SDE, the drift term depends not only on $Y^\alpha_t$, but on the whole trajectory $(Y^\alpha_s)_{s\leq t}$,
so the process is not Markovian. 
\vspace{2mm}\\Still, due to results from \cite{brunick_shreve}, the SDE \eqref{eq_SDE_general} admits a solution $\widehat{Y}^\alpha$ whose marginals coincide with $Y^\alpha$, e.g.\ $\textnormal{Law}(\widehat{Y}_t^\alpha)=\textnormal{Law}(Y_t^\alpha)$ holds for all $0\leq t <T_{\textnormal{gen}}$.
Hence, the SDE \eqref{eq_SDE_general} can still be utilized for sampling and
simulation is then achieved identically to the standard case, as discussed in sections 3 and 4 of \cite{SLIPS}.
Since the derivation of the SDE descriptions of $Y^\alpha$ and $\widehat{Y}^\alpha$ are discussed only briefly in a recently  arXiv-updated version of \cite{SLIPS}, we discuss them in more detail in Appendix \ref{appendix_SDE_theory}.
\vspace{2mm}\\The remarkable observation made in \cite{SLIPS} is that the choice of denoising schedule has a negligible effect on the performance of SLIPS, provided the discretization is chosen adapted to it. To make a suitable choice,
Grenioux et al.\ use the signal-to-noise-ratio (SNR), which is defined as
\begin{align}
    \textnormal{SNR}^\alpha_t \coloneqq \frac{\mb{E}\big[\|\alpha(t)X-\mb{E}[\alpha(t)X]\|^2\big]}{\mb{E}[\|\sigma B_t\|^2]}
    \propto \frac{\alpha(t)^2}{t}.
\end{align}
Given $0<t_0<t_K$, they propose choosing the discretization $(t_k)_{0\leq k \leq K}$ between $t_0$ and $t_K$ that is uniquely defined by asking that for some $\Delta>0$
\begin{align}\label{eq_condition_SNR}
    \textnormal{log} \,\textnormal{SNR}^\alpha_{t_{k+1}}
    = \textnormal{log} \,\textnormal{SNR}^\alpha_{t_{k}}+\Delta,
\end{align}
which is referred to as the \textit{log-SNR adapted discretization}. 
Observe that this discretization yields a finer grid in time regions where the denoising procedure accelerates faster.
Notably, the log-SNR adapted discretization does not only lead to equivalent performance of different denoising schedules, for any given denoising schedule it is also the optimally performing discretization. Hence, it is chosen as the default discretization in SLIPS.
\vspace{2mm}\\Due to the experimentally observed equivalence, we restrict our theoretical analysis to the case of the standard SL observation process.
\vspace{2mm}\\Observe that, considering the standard denoising schedule $\alpha(t)=t$, the condition
\eqref{eq_condition_SNR}
translates into
\begin{align}
    t_{k+1} = t_{k} e^\Delta.
\end{align}
In particular, we thus have $\Delta = \frac{1}{K}\log \frac{t_K}{t_0}$ and $(t_k)_{0\leq k \leq K}=\big((\frac{t_K}{t_0})^{\frac{k}{K}}t_0\big)_{0\leq k \leq K}$.
Further, when comparing different denoising schedules, to ensure comparability, instead of giving the final integration time $T=t_K\in (t_0,T_{\textnormal{gen}})$ as input to SLIPS, one chooses a prespecified SNR level $\eta>0$. As we consider only SLIPS with standard denoising schedule $\alpha(t)=t$ in our analysis, we choose as input directly $T>0$.
\vspace{2mm}\\Regarding the choice of $\sigma$, to make the SNR independent of the target and to reduce the number of tuning parameters, Grenioux et al.\ propose to set
\begin{align}\label{eq_sigma}
    \sigma^2 \coloneqq \frac{R^2_{\pi}}{d},
\end{align}
where $R^2_{\pi}\coloneqq\mb{E}\big[\|X-\mb{E}[X]\|^2\big]$. Indeed, in that case $\textnormal{SNR}^\alpha_t =\alpha(t)^2/t$. If $R^2_{\pi}$ is not known, one can replace it with an estimate $\hat{R}^2_\pi$. For simplicity of exposition, we will assume in our analysis that we know $R^2_{\pi}$. Since \eqref{eq_sigma} makes sense only when $\pi \in \mc{P}_2(\mb{R}^d)$, we make the following assumption.
\begin{assumption}\label{assumption_moment}
    We have $\pi\in \mc{P}_2(\mb{R}^d)$.
\end{assumption}

The SLIPS algorithm with standard denoising schedule and log-SNR adapted discretization is stated as Algorithm \ref{alg:slips}.

\subsection{Stochastic Localization and SGM}
We now give some details on score-based generative models (SGMs) and motivate why techniques used to study SGMs are useful for studying SLIPS.
Recall that in an SGM, one uses a diffusion process to sample from a target $\pi$, of which one has a large number of samples but no prior information about the density. The standard choice of diffusion process is the time reversal of the Ornstein Uhlenbeck (OU) process started in $\pi$, see e.g.\ \cite{GM_song2021scorebased}. We remind ourselves that the OU process started in $\pi$ is the unique solution to the SDE
\begin{align}
    \begin{cases}\label{eq_OU}
        dZ_t &= -Z_t\,dt+ \sqrt{2}\,dB_t,
        \\ Z_0&\sim \pi,
    \end{cases}
\end{align}
where the Brownian motion $B$ and $Z_0$ are independent. Setting $p_t^{\textnormal{OU}}\coloneqq \textnormal{Law}(Z_t)$, for $T>0$ the time reversal of $Z$ on $[0,T]$ solves
\begin{align}
    \begin{cases}\label{eq_OU_reversed}
        d\overset{\leftarrow}{Z_t} &= \overset{\leftarrow}{Z_t}+2\nabla \log p_{T-t}^{\textnormal{OU}}(\overset{\leftarrow}{Z_t})\,dt+ \sqrt{2}\,d\bar{B}_t,
        \\ \overset{\leftarrow}{Z_0}&\sim p_{T}^{\textnormal{OU}},
    \end{cases}
\end{align}
see \cite{time_reversal}, \cite{time_reversal_diffusion}.
Here, the Brownian motion $\bar{B}$ and $\overset{\leftarrow}{Z_0}$ are again independent. Generating samples from $\pi$ is then achieved by simulating \eqref{eq_OU_reversed}, replacing the score with an estimate learned from the data via neural networks, using techniques such as score-matching \cite{score_matching_article}.
\vspace{2mm}\\Since the OU process admits the explicit representation
\begin{align}
    Z_t = Z_0\,e^{-t}+ \sqrt{2}\, \int_0^t\,e^{-(t-s)}\,dB_s,
\end{align}
we thus have 
\begin{align}
    p_t^{\textnormal{OU}} = 
    \big((e^{-t})_\#\pi\big) \ast \mc{N}\big(0,1-e^{-2t}\big).
\end{align}
Hence, conceptually the OU process is simply an interpolation between Gaussian noise and the target by convolution, up to the scaling. The time reversal is then the converse interpolation.
\vspace{2mm}\\In particular, considering \eqref{eq_p}, the reverse OU process and the SL observation process both interpolate between a Gaussian and the target in the same way, just on different time and space scales. 
Further, on a formal level, plugging the Tweedie formula \eqref{eq_tweedie} into the SL SDE \eqref{eq_SDE_Y}, we find that the structure of the SDE is identical to the one of \eqref{eq_OU_reversed}.
\vspace{2mm}\\Thus, it is natural to apply techniques developed for the study of SGMs to the setting of SLIPS.

\section{Sampling guarantees in  TV distance}\label{section_TV}

We now prepare to state our main result, a theoretical guarantee for SLIPS in TV distance. We make the following assumption, analogous to the $L^2$ score-estimation assumption in SGM \cite{Sampling_is_as_easy_as_learning_the_score_chen2023sampling}, \cite{benton2024nearly}. 
\begin{assumption}\label{assumption_posterior_estimator_SL_process}
    For $\varepsilon_0>0$, writing $\delta_k=t_{k+1}-t_k$, the MALA-based posterior expectation estimators $\{\hat{u}_{t_k}\}_{0\leq k \leq K-1}$
    fulfill 
    \begin{align}\notag
        \frac{1}{t_K-t_0}\sum_{0\leq k \leq K-1} \delta_k\big\|
        \hat{u}_{t_k} (Y_{t_k})-
        u_{t_k} (Y_{t_k})
        \big\|_{L^2} \leq \varepsilon_0,
    \end{align}
    where $Y$ denotes the SL observation process.
\end{assumption}
Note that our setup is in line with the state of the art setup for SGMs \cite{benton2024nearly}: We impose an $L^2$-assumption on the posterior expectation estimator and do not demand the score of $p_t$ to be Lipschitz. 
In earlier works on SGMs, one imposed $L^\infty$ accuracy assumptions and demanded the score being Lipschitz \cite{sgm_lee}, \cite{sgm_lee_2}, \cite{Sampling_is_as_easy_as_learning_the_score_chen2023sampling}, which we recall translates directly to assumptions on the posterior expectation estimator and vice versa due to analogues of \eqref{eq_tweedie} for the OU process.
\vspace{2mm}\\In the following, for $0\leq k \leq K$ we write
\begin{align}
    \tilde{p}_{t_k} \coloneqq \textnormal{Law}(\widetilde{Y}_{t_k}).
\end{align}

\subsection{Main results}
The following is our main result, a bound on the TV distance between the SLIPS output and the target. 
\begin{theorem}\label{theorem_slips_tv_guarantee}
    Assume Assumptions \ref{assumption_moment} and \ref{assumption_posterior_estimator_SL_process}, let $K\in \mb{N}$, $0<t_0<t_K$, and write $T=t_K$.
    Consider an arbitrary discretization $(t_k)_{0\leq k \leq K}$ of $[t_0,T]$.
    Then, the law $\tilde{\pi}_{T} \coloneqq (1/T)_\# \tilde{p}_T$ of the output of SLIPS run with the chosen discretization fulfills
    \begin{align}\notag 
        d_{\textnormal{TV}}(\tilde{\pi}_{T},\pi)\leq &
        d_{\textnormal{TV}}(\tilde{p}_{t_0},p_{t_0})+
        \sqrt{d\cdot C_{\textnormal{disc}}}
        \\ \notag &
        +\sqrt{\frac{1}{\sigma^2} T \varepsilon_0^2}
        +\frac{1}{2}\|\nabla \log \pi\|_{L^2(\pi)} \,\sqrt{d\frac{\sigma^2}{T}},
    \end{align}
    where $\varepsilon_0$ is from Assumption \ref{assumption_posterior_estimator_SL_process} and
    \begin{align}\notag
        C_{\textnormal{disc}}\coloneqq &
        \sum_{1\leq k\leq K-1} \max\big\{0,(t_{k+1}-t_k)-(t_{k}-t_{k-1})\big\}
        \frac{1}{t_k}
        \\ \label{eq_sum_discretization} &+ \frac{t_1-t_0}{t_0}.
    \end{align}
\end{theorem}
As a consequence of the Theorem, one obtains guarantees for the order of $K$ necessary to obtain an $\varepsilon$-guarantee in TV distance when using the log-SNR adapted discretization that are linear in $d$. This is the content of Corollary \ref{corollary_kL}. We state an informal version below, a formal version can be found in Appendix \ref{appendix_main_results} as Corollary \ref{corollary_kL_appendix}.
\begin{corollary}[Informal]\label{corollary_kL}
    Let $\varepsilon>0$. In the setup of Theorem \ref{theorem_slips_tv_guarantee}, assume $R^2_{\pi},\|\nabla \log \pi \|^2_{L^2(\pi)} \in \mc{O}(d)$.
    \vspace{2mm}\\Then, for SLIPS run with log-SNR adapted discretization, taking $T$ of order $\frac{1}{\varepsilon^2}\|\nabla \log \pi \|^2_{L^2(\pi)} \,R^2_{\pi}$ and $K$ of order $d/\varepsilon^2$ suffices to ensure $d_{\textnormal{TV}}(\tilde{\pi}_T,\pi)\leq \varepsilon$, provided $\varepsilon_0$ and $d_{\textnormal{TV}}(\tilde{p}_{t_0},p_{t_0})$ are sufficiently small.
\end{corollary}
\begin{remark}
    In an ideal case, we would not assume a priori that the initialization error and the error in the estimation of the posterior expectation are sufficiently small. Instead, we would wish to derive bounds on them depending on the number of MCMC iterations. 
    This is in principle feasible using results from \cite{DALALYAN_inaccurate_gradient}, \cite{MALA_geometric_ergodicity}, \cite{Brosse2024_MCMC_control_variates}. Though, to obtain an $\varepsilon$-guarantee, the implied overall computational complexity scales exponentially with the dimension. Thus, the result would not give a meaningful explanation for the strong empirical performance of SLIPS.
\end{remark}
\begin{remark}\label{remark_example}
    Note that the condition of Corollary \ref{corollary_kL} is not restrictive. Indeed, in the context of a bi-modal GMM with means $m_1,m_2\in \mb{R}^d$, fixed variance $\sigma^2>0$ and weights $w,1-w$ for $w\in [0,1]$, i.e.\ $\pi(x)=w \,\mc{N}(x,m_1,\sigma^2)+(1-w)\, \mc{N}(x,m_2,\sigma^2)$, a direct computation of $R^2_{\pi}$ yields
    \begin{align}
        R^2_{\pi} = d\sigma^2 +w(1-w) \|m_1-
        m_2\|^2.
    \end{align}
    Further, for $\|\nabla \log \pi \|^2_{L^2(\pi)}$ we find by computation that
    \begin{align}
        \|\nabla \log \pi \|_{L^2(\pi)}^2
        \leq \frac{d}{\sigma^2}.
    \end{align}
    Hence, our imposed condition holds provided $\sigma^2 \in \Theta(1)$ and $\|m_1-
        m_2\|^2\in \mc{O}(d)$.
\end{remark}

\subsection{Proof strategy}\label{section_sub_proof_strategy}
We now discuss the strategy for proving Theorem \ref{theorem_slips_tv_guarantee}. Full details can be found in Appendix \ref{appendix_main_results}. We study separately the information and discretization error. By information error we refer to the error that is due to $Y_t/t$ being at any time $t<\infty$ only a noisy sample from $\pi$, whereas by discretization error we refer to the error that occurs due to our 
discretization of the SDE \eqref{eq_SDE_Y}. 
\vspace{2mm}\\For the information error, we have the following bound.
\begin{lemma}\label{lemma_tv_approx_SL}
    Setting $\pi_t \coloneqq (1/t)_\# p_t$, for all $t\in (0,\infty)$ it 
    holds that
    \begin{align}
        d_{\textnormal{TV}}(\pi,\pi_{t})
        \leq
        \frac{1}{2}
        \|\nabla \log \pi\|_{L^2(\pi)} \,\sqrt{d\frac{\sigma^2 }{t}}.
    \end{align}
\end{lemma}
The idea of the proof of Lemma \ref{lemma_tv_approx_SL} is to exploit first the characterization of the TV distance as the $L^1$ distance between the densities. Then, one uses that the density of $\pi_t$ corresponds to a smoothed version of $\pi$ obtained by convolution with a Gaussian kernel and combines this with standard estimates.
\vspace{2mm}\\We establish the following bound on the discretization error.
\begin{lemma}\label{lemma_discretization_error}
    Assume Assumptions \ref{assumption_moment} and \ref{assumption_posterior_estimator_SL_process}. Consider SLIPS run with an arbitrary discretization $(t_k)_{0\leq k \leq K}$ of $[t_0,t_K]$.
    Then
    \begin{align}\label{eq_discretization_lemma}
        d_{\textnormal{TV}}(\tilde{p}_{t_K},p_{t_K}) 
        \leq &
        d_{\textnormal{TV}}(\tilde{p}_{t_0},p_{t_0})+
        \sqrt{d\cdot C_{\textnormal{disc}}}
        \\ \notag & +\sqrt{\frac{1}{\sigma^2} t_K \varepsilon_0^2},
    \end{align}
    where $C_{\textnormal{disc}}$ is given in \eqref{eq_sum_discretization}.
\end{lemma}
The proof of Lemma \ref{lemma_discretization_error} relies on techniques developed in the context of SGMs. 
The key observation to prove Lemma \ref{lemma_discretization_error} is that, even though SLIPS is defined by the recursion \eqref{eq_EM_recursion}, we can equivalently obtain the $k+1$-st SLIPS iterate from the $k$-th one by solving the SDE
\begin{align}
    d\widetilde{Y}_t = \hat{u}_{t_k}(\widetilde{Y}_{t_k})\,dt+ \sigma \,dB_t
\end{align}
on $[t_k,t_{k+1})$.
This allows us to compare the output of the (discrete) SLIPS algorithm with the normalized (continuous) SL observation process by comparing measures on path space associated to solutions of SDEs, instead of directly studying the marginal laws $\tilde{p}_{t_K}$ and $p_{t_K}$. The benefit of this is that, switching from TV to Kullback-Leibler (KL) divergence by Pinsker's inequality, comparing laws of solutions to SDEs becomes tractable using tools such as Girsanov's Theorem.
\vspace{2mm}\\The analysis then boils down to bounding
\begin{align}\label{eq_sum_disc_error}
    \sum_{0\leq k \leq K-1}\int_{t_k}^{t_{k+1}}\,\mb{E}\big[\|u_t(Y_t)-u_{t_k}(Y_{t_k})\|^2\big],
\end{align}
which is done following an argument developed in \cite{benton2024nearly} for the reverse Ornstein Uhlenbeck process.
\vspace{2mm}\\An important ingredient that allows us to establish a bound on \eqref{eq_sum_disc_error} in a manner more straightforward than the analogous results for SGMs is that $u_t(Y_t)=\mb{E}[X|Y_t]$ is a \textit{martingale} with respect to the filtration $\mc{F}^Y_t=\sigma\big((Y_s)_{s\in [0,t]}\big)$. Indeed, the conditional distribution of $X$ given $Y_t$ coincides with the conditional distribution of $X$ given $\{Y_s:s\in [0,t]\}$. Hence, the conditional expectations coincide. Since $\mb{E}[X|(Y_s)_{s\in [0,t]}]$ is a martingale, it follows that $u_t(Y_t)$ is one. To prove that the conditional distributions coincide, we show that for all finite time marginals $0<s_1<s_2<\dots<s_k=t$ we have
\begin{align}
    \mb{P}(X\in \cdot\,|\{Y_{s_l}\}_{0\leq l \leq k})=
    \mb{P}(X\in \cdot\,|Y_{t}).
\end{align}
This follows by computing the left hand side explicitly via Bayes Theorem, using that, conditional on $X$, the marginals $\{Y_{s_{l+1}}-Y_{s_l}\}_{0\leq l \leq k-1}$ are independent. For the details, we refer to Lemma \ref{lemma_posterior_exp_simplification} in Appendix \ref{appendix_SDE_theory}. There is also an alternative SDE-based way to see that $u_t(Y_t)$ is a martingale, more in the flavour of Stochastic Localization, which is discussed and used in the proof of Lemma \ref{lemma_u_L2}.
\begin{remark}
    In the case of a general denoising schedule, our proof approach is still applicable, though less straightforward. 
    Indeed, the process $u_t^{\alpha}(\widetilde{Y}_t^\alpha)$ is in general no martingale if $\alpha$ is non-linear, so that bounding the analogue of \eqref{eq_sum_disc_error} is more complicated. Still, it is feasible using SDE techniques, see also Remark \ref{remark_sum_bound} in Appendix \ref{appendix_main_results}.
\end{remark}

\begin{algorithm}[ht]
   \caption{SLIPS (with log-SNR adapted discretization)}
   \label{alg:slips}
\begin{algorithmic}
   \STATE {\bfseries Input:} $t_0$, $T$, $K$, $M$ 
   \STATE Set $\sigma=R_\pi/\sqrt{d}$, see \eqref{eq_sigma}
   \STATE Set $(t_k)_{k=0}^K$ as the log-SNR-adapted disc.\ of $[t_0, T]$, see \eqref{eq_condition_SNR}
   \STATE Initialize $\widetilde{Y}_{t_0}$ with \Cref{alg:langevin_within_langevin}
   \FOR{$k=0$ {\bfseries to} $K-1$}
   \STATE Define $\delta_k=t_{k+1}-t_k$
   \STATE Simulate $\{X^k_j\}_{j=1}^M \sim q_{t_k}(\cdot|\widetilde{Y}_{t_k})$ with MALA
   \STATE Estimate the denoiser by $\hat{u}_{t_k}(\widetilde{Y}_{t_k})=1/M\sum_{j=1}^M X_j^k$
   \STATE Simulate $\widetilde{Y}_{t_{k+1}} \sim \mc{N}(\widetilde{Y}_{t_{k}} + \delta_k \hat{u}_{t_k}(\widetilde{Y}_{t_k}), \sigma^2 \delta_k)$
   \ENDFOR
   \STATE {\bfseries Output:} $\widetilde{Y}_{t_{K}}/t_K$
\end{algorithmic}
\end{algorithm}

\begin{algorithm}[ht]
   \caption{Langevin-within-Langevin initialization}
   \label{alg:langevin_within_langevin}
    \begin{algorithmic}
       \STATE {\bfseries Input:} $t_0$, $\sigma$, $N$, $M$
       \STATE Set $Y^{(0)} \sim \mc{N}(0, \sigma^2 t_0 \, \Idd)$ and $\lambda = \sigma^2 t_0 / 2$
       \FOR{$n = 0$ {\bfseries to} $N-1$}
            \STATE Simulate $\{X^{(n)}_j\}_{j=1}^M \sim q_{t_0}(\cdot | Y^{(n)})$ with MALA
            \STATE Estimate the denoiser by $U^{(n)}= (1/M) \sum_{j=1}^M X^{(n)}_j$
            \STATE Set $\hat{s}_{t_0}(Y^{(n)}) = \left(t_0 U^{(n)} - Y^{(n)}\right) / (\sigma^2 t_0)$, see \eqref{eq_tweedie}
            \STATE Simulate $Y^{(n+1)} \sim \mc{N}(Y^{(n)} + \lambda \hat{s}_{t_0}(Y^{(n)}), 2 \lambda \, \Idd)$
       \ENDFOR
       \STATE {\bfseries Output:} $Y^{(N)}$
    \end{algorithmic}
\end{algorithm}

\section{Optimal choice of discretization}\label{section_discretization}
We now give theoretical insights into the empirically observed optimality of the log-SNR adapted discretization.
In the bound \eqref{eq_discretization_lemma} on the discretization error from
Lemma \ref{lemma_discretization_error}, the right hand side depends on the choice of discretization solely through $C_{\textnormal{disc}}$ defined in \eqref{eq_sum_discretization}.
Our next result establishes that the log-SNR adapted discretization minimizes $C_{\textnormal{disc}}$ in a conditional sense.
The proof can be found in Appendix \ref{appendix_discretization}.
\begin{lemma}\label{lemma_opt_step_size}
    Let $0<t_0<t_K$ be fixed.
    Among all choices of $K-1$ discretization points $(t_k)_{1\leq k\leq K-1}$ between $t_0$ and $t_K$ with $t_{K-1}=t_0\big(\frac{t_K}{t_0}\big)^{(K-1)/K}$, the log-SNR adapted choice $t_k=t_0\big(\frac{t_K}{t_0}\big)^{k/K}$ minimizes $C_{\textnormal{disc}}$ defined in \eqref{eq_sum_discretization}.
\end{lemma}
\begin{remark}
    In Lemma \ref{lemma_opt_step_size}, ideally we would wish to not impose any assumption on $t_{K-1}$. We do so for technical reasons, discussed in Appendix \ref{appendix_discretization}.
    Observe though that the Lemma is still relevant from a practical perspective. Indeed, in all relevant scalings the number of discretization points diverges. Further, one expects that the exact choice of the $K-1$-st discretization point matters little, as long as it is sufficiently close to the $K-2$-nd and the last discretization point.
\end{remark}
\begin{remark}\label{remark_c_disc}
    A natural question is whether $C_{\textnormal{disc}}$ depends only negligibly on the choice of discretization or whether the precise choice affects the order of it. The latter is the case, as we illustrate now. The canonical choice of step-size for any SDE discretization is the uniform, e.g.\ $t_{k+1}-t_k=\frac{t_{K}-t_0}{K}$. 
    In that case the sum in \eqref{eq_sum_discretization} vanishes, so we have 
    \begin{align}
        C_{\textnormal{disc}} =
        \frac{t_K-t_0}{Kt_0}\sim \frac{t_K}{Kt_0}.
    \end{align}
    In contrast, setting $\kappa \coloneqq e^\Delta-1=e^{\frac{1}{K}\log \frac{t_K}{t_0}}-1$, for the log-SNR adapted discretization we have
    \begin{align}
        C_{\textnormal{disc}} =\kappa^2 K \frac{1}{1+\kappa}+\kappa \sim \kappa (\kappa K +1).
    \end{align}
    Due to $\kappa \sim \frac{1}{K}\log \frac{t_K}{t_0}$, we find that $C_{\textnormal{disc}}$ scales linearly in $\frac{t_K}{t_0}$ if the discretization is chosen uniformly, whereas it scales logarithmically in $\frac{t_K}{t_0}$ in the log-SNR adapted case. Recall from Corollary \ref{corollary_kL} and Lemma \ref{lemma_tv_approx_SL} that $t_K$ has a hidden dimension dependence through $R^2_{\pi}$ and $\|\nabla \log \pi\|_{L^2(\pi)}^2$. Further, in practice $t_0$ degenerates as $d\to \infty$. Hence, the scaling of $C_{\textnormal{disc}}$ under the log-SNR adapted discretization is by an order of magnitude better than in the uniform step-size case.
\end{remark}

\section{SLIPS vs.\ RDMC}\label{section_slips_rdmc}
Huang et al.\ proposed an approach for sampling from a distribution $\pi$ for which one can access the unnormalized density that is conceptually closely related to SLIPS. The proposed method is coined Reverse Diffusion Monte Carlo \cite{huang2024reverse}. They aim to sample from $\pi$ by simulating the reverse Ornstein Uhlenbeck process associated to $\pi$, see \eqref{eq_OU_reversed}. Similar to SLIPS, they estimate the drift by sampling from a posterior measure. To do so, they have to additionally employ a Tweedie formula analogous to \eqref{eq_tweedie} to express the score in terms of a posterior expectation, see equation $(4)$ in \cite{huang2024reverse}. For the initialization, they also use a Langevin-within-Langevin type algorithm. 
\vspace{2mm}\\Notably, in contrast to SLIPS, Huang et al.\ derived theoretical convergence guarantees for their approach, see Section 4.1 in \cite{huang2024reverse}.
In the following, we highlight differences and similarities of those guarantees compared to the ones we derived, see Table \ref{tab:algorithm_comparison} for an overview. For a general comparison of SLIPS and RDMC, we refer to Appendix G of \cite{SLIPS}.
\vspace{2mm}\\First of all, Huang et al.\ make the assumption that the score of $p_t^{\textnormal{OU}}$ is globally Lipschitz in space uniformly across $t\in[0,T]$. Considering $t=0$, this includes in particular the requirement that the score of $\pi$ is globally Lipschitz. We do not impose any Lipschitz requirements in our analysis of SLIPS. 
\begin{table}[ht]
\centering
\caption{Comparison of SLIPS and RDMC guarantees.}
\label{tab:algorithm_comparison}
\begin{tabular}{lcc}
\toprule
\textbf{} & \textbf{SLIPS} & \textbf{RDMC} \\
\midrule
Lipschitz assumption on scores      & $\times$                 & $\checkmark$                 \\
Order of $K$ for $\varepsilon$-guarantee & $d/\varepsilon^2$                  & $dL^2/\varepsilon^2$                   \\
Guarantees for MCMC methods &             $\times$      & $\checkmark$                \\
\bottomrule
\end{tabular}
\end{table}
Denoting the uniform bound on the Lipschitz constant by $L$, Huang et al.\ prove that taking $K$ of order $dL^2/\varepsilon^2$ suffices to achieve an $\varepsilon$ guarantee in TV distance, see Section 4.1 in \cite{huang2024reverse}.
While the explicit dependence of $K$ on the dimension $d$ is linear, there is a hidden dependence on the dimension through the quadratic dependence on $L$. In many cases of interest the Lipschitz constant of the score of $\pi$ scales at least linearly with the dimension. In that case, the number of discretization steps $K$ has to be chosen of order $d^3$, which is substantially worse than our linear guarantee. One simple example in which this is the case is the balanced bi-modal GMM with density $\pi(x)=\frac{1}{2}(\mc{N}(x,\textbf{1},\sigma^2)+\mc{N}(x,-\textbf{1},\sigma^2))$, where $\textbf{1}\in \mb{R}^d$ is the vector with all ones and $\sigma^2>0$ is fixed. 
\vspace{2mm}\\Further, Huang et al.\ do not impose a priori that the Monte Carlo based posterior expectation estimators are well behaved. Instead, they assume a priori that the posterior measures fulfill log-Sobolev inequalities (LSI) and prove, conditional on the LSIs and the Markov chains being run for sufficiently long time, probabilistic guarantees for the posterior expectation estimator. 
\vspace{2mm}\\A central difference though lies in the fact that they assume that the posterior expectation estimator is the empirical average of the final iterates of $M$ \textit{independently} run Markov chains. Recall that in a given iteration of SLIPS we run only a single Markov chain and take the ergodic average of its iterates as an estimate of the posterior expectation, which is vastly more efficient computationally. In particular, their argument does not translate to our setting, since they face the problem of studying the concentration of i.i.d.\ random variables, which is fundamentally different to studying the concentration of ergodic averages of Markov chains.  

\section{Conclusion}\label{section_conclusion}
In this work, we derived a novel guarantee for SLIPS in the total variation distance, using tools from score-based generative modeling. Notably, our guarantee implies a linear dependence of the number of discretization points necessary to achieve an $\varepsilon$-guarantee on the dimension.
We studied the complexity of SLIPS conditional on well behaved initialization and posterior expectation estimates. 
It is of interest to derive new results on the estimation of expectations via ergodic averages of MALA that do not suffer from an exponential dimension dependence and, hence, can be incorporated into our analysis of SLIPS.

\section*{Acknowledgements}

The author gratefully acknowledges financial support
by the Deutsche Forschungsgemeinschaft (DFG) through the IRTG 2544 “Stochastic Analysis in Interaction”. The author thanks Christian Bayer, Andreas Eberle, Louis Grenioux, and Maxence Noble for helpful exchange.

\bibliographystyle{alpha}
\bibliography{slips_paper}

@inproceedings{SLIPS,
author = {Grenioux, Louis and Noble, Maxence and Gabri\'{e}, Marylou and Durmus, Alain},
title = {{Stochastic localization via iterative posterior sampling}},
year = {2024},
publisher = {JMLR.org},
booktitle = {{Proceedings of the 41st International Conference on Machine Learning}},
articleno = {651},
numpages = {40},
location = {Vienna, Austria},
series = {ICML'24}
}

@article{Brosse2024_MCMC_control_variates,
  author  = {Brosse, N. and Durmus, A. and Meyn, S. and Moulines, E. and Samsonov, S.},
  title   = {{Diffusion Approximations and Control Variates for MCMC}},
  journal = {Computational Mathematics and Mathematical Physics},
  volume  = {64},
  number  = {4},
  pages   = {693--738},
  year    = {2024},
  doi     = {10.1134/S0965542524700167},
  issn    = {1555-6662}
}

@article{MALA_geometric_ergodicity,
    author = {Oliviero-Durmus, Alain and Moulines, Éric},
    title = {{On geometric convergence for the Metropolis-adjusted Langevin algorithm under simple conditions}},
    journal = {Biometrika},
    volume = {111},
    number = {1},
    pages = {273-289},
    year = {2023},
    month = {10},
    issn = {1464-3510},
    doi = {10.1093/biomet/asad060},
    url = {https://doi.org/10.1093/biomet/asad060}
}

@book{Liptser2001,
  author    = {Liptser, Robert S. and Shiryaev, Albert N.},
  title     = {{Statistics of Random Processes: I. General Theory}},
  series    = {Stochastic Modelling and Applied Probability},
  edition   = {2},
  publisher = {Springer Berlin, Heidelberg},
  year      = {2001},
  isbn      = {978-3-540-63929-9},
  doi       = {10.1007/978-3-662-13043-8},
  url       = {https://doi.org/10.1007/978-3-662-13043-8},
  note      = {Original Russian edition published by Nauka, Moscow, 1974.}
}

@article{DALALYAN_inaccurate_gradient,
title = {{User-friendly guarantees for the Langevin Monte Carlo with inaccurate gradient}},
journal = {Stochastic Processes and their Applications},
volume = {129},
number = {12},
pages = {5278-5311},
year = {2019},
issn = {0304-4149},
doi = {https://doi.org/10.1016/j.spa.2019.02.016},
url = {https://www.sciencedirect.com/science/article/pii/S0304414918304824},
author = {Arnak S. Dalalyan and Avetik Karagulyan}
}

@inproceedings{
Sampling_is_as_easy_as_learning_the_score_chen2023sampling,
title={{Sampling is as easy as learning the score: theory for diffusion models with minimal data assumptions}},
author={Sitan Chen and Sinho Chewi and Jerry Li and Yuanzhi Li and Adil Salim and Anru Zhang},
booktitle={{The Eleventh International Conference on Learning Representations }},
year={2023},
url={https://openreview.net/forum?id=zyLVMgsZ0U_}
}

@misc{montanari2023samplingdiffusionsstochasticlocalization,
      title={{Sampling, Diffusions, and Stochastic Localization}}, 
      author={Andrea Montanari},
      year={2023},
      archivePrefix={arXiv},
      url={https://arxiv.org/abs/2305.10690}, 
}

@book{chewi2024logconcave,
  title     = {{Log-Concave Sampling}},
  author    = {Sinho Chewi},
  year      = {2024},
  note      = {Book draft available at \url{https://chewisinho.github.io/main.pdf}},
}

@INPROCEEDINGS{SL_chen_eldan,
  author={Chen, Yuansi and Eldan, Ronen},
  booktitle={{2022 IEEE 63rd Annual Symposium on Foundations of Computer Science (FOCS)}}, 
  title={{Localization Schemes: A Framework for Proving Mixing Bounds for Markov Chains}}, 
  year={2022},
  volume={},
  number={},
  pages={110-122},
  doi={10.1109/FOCS54457.2022.00018}}

@article{roberts_tweedie_langevin,
author = {Gareth O. Roberts and Richard L. Tweedie},
title = {{Exponential convergence of Langevin distributions and their discrete approximations}},
volume = {2},
journal = {Bernoulli},
number = {4},
publisher = {Bernoulli Society for Mathematical Statistics and Probability},
pages = {341 -- 363},
year = {1996},
}

@book{karatzas1991brownian,
  title        = {{Brownian Motion and Stochastic Calculus}},
  author       = {Ioannis Karatzas and Steven E. Shreve},
  year         = {1991},
  series       = {Graduate Texts in Mathematics},
  volume       = {113},
  publisher    = {Springer},
  address      = {New York, NY},
  doi          = {10.1007/978-1-4612-0949-2},
  isbn         = {978-0-387-97655-6},
  edition      = {2},
}

@misc{eberle2024mcmc,
  title        = {{Markov Chain Monte Carlo Methods}},
  author       = {Nawaf Bou-Rabee and Andreas Eberle},
  year         = {2024},
  month        = {November},
  day          = {11},
  note         = {Lecture notes, Universit{\"a}t Bonn},
  url          = {https://wt.iam.uni-bonn.de/fileadmin/WT/Inhalt/people/Andreas_Eberle/Teaching/MCMC_WS2024_2025/Lecture_Notes_MCMC.pdf}
}

@inproceedings{
GM_song2021scorebased,
title={{Score-Based Generative Modeling through Stochastic Differential Equations}},
author={Yang Song and Jascha Sohl-Dickstein and Diederik P Kingma and Abhishek Kumar and Stefano Ermon and Ben Poole},
booktitle={{International Conference on Learning Representations}},
year={2021},
url={https://openreview.net/forum?id=PxTIG12RRHS}
}

@article{time_reversal,
author = {Haussmann, U. G. and Pardoux, Étienne},
title = {{Time Reversal of Diffusions}},
volume = {14},
journal = {The Annals of Probability},
number = {4},
publisher = {Institute of Mathematical Statistics},
pages = {1188 -- 1205},
year = {1986},
doi = {10.1214/aop/1176992362},
URL = {https://doi.org/10.1214/aop/1176992362}
}

@article{Eldan2013,
  author    = {Ronen Eldan},
  title     = {{Thin Shell Implies Spectral Gap Up to Polylog via a Stochastic Localization Scheme}},
  journal   = {Geometric and Functional Analysis},
  year      = {2013},
  volume    = {23},
  number    = {2},
  pages     = {532--569},
  doi       = {10.1007/s00039-013-0214-y},
  url       = {https://doi.org/10.1007/s00039-013-0214-y},
  issn      = {1420-8970}
}

@article{Eldan2020,
author = {Eldan, Ronen},
year = {2020},
pages = {},
title = {{Taming correlations through entropy-efficient measure decompositions with applications to mean-field approximation}},
volume = {176},
journal = {Probability Theory and Related Fields},
doi = {10.1007/s00440-019-00924-2}
}

@INPROCEEDINGS{montanari_SK_model_2022,
  author={Alaoui, Ahmed El and Montanari, Andrea and Sellke, Mark},
  booktitle={{2022 IEEE 63rd Annual Symposium on Foundations of Computer Science (FOCS)}}, 
  title={{Sampling from the Sherrington-Kirkpatrick Gibbs measure via algorithmic stochastic localization}}, 
  year={2022},
  pages={323-334},
  doi={10.1109/FOCS54457.2022.00038}
}

@inproceedings{
    benton2024nearly,
    title={{Nearly d-Linear Convergence Bounds for Diffusion Models via Stochastic Localization}},
    author={Joe Benton and Valentin De Bortoli and Arnaud Doucet and George Deligiannidis},
    booktitle={{The Twelfth International Conference on Learning Representations}},
    year={2024},
    url={https://openreview.net/forum?id=r5njV3BsuD}
    }

@inproceedings{
huang2024reverse,
title={{Reverse Diffusion Monte Carlo}},
author={Xunpeng Huang and Hanze Dong and Yifan Hao and Yian Ma and Tong Zhang},
booktitle={{The Twelfth International Conference on Learning Representations}},
year={2024},
url={https://openreview.net/forum?id=kIPEyMSdFV}
}

@inproceedings{
richter2024improved,
title={{Improved sampling via learned diffusions}},
author={Lorenz Richter and Julius Berner},
booktitle={{The Twelfth International Conference on Learning Representations}},
year={2024},
url={https://openreview.net/forum?id=h4pNROsO06}
}

@book{durrett_probability,
author = {Durrett, Rick},
title = {{Probability: Theory and Examples}},
year = {2019},
isbn = {0521765390},
publisher = {Cambridge University Press},
address = {USA},
edition = {5th edition}
}

@book{mixture_models_book,
publisher = {John Wiley \& Sons, Ltd},
isbn = {9780471721185},
title = {{Finite Mixture Models}},
url = {https://onlinelibrary.wiley.com/doi/abs/10.1002/0471721182.fmatter},
author={Geoffrey McLachlan and David Peel},
year = {2000}
}

@article{sequential_MC,
  author    = {Pierre Del Moral and Arnaud Doucet and Ajay Jasra},
  title     = {{Sequential Monte Carlo samplers}},
  journal   = {Journal of the Royal Statistical Society: Series B (Statistical Methodology)},
  volume    = {68},
  number    = {3},
  pages     = {411--436},
  year      = {2006},
  doi       = {10.1111/j.1467-9868.2006.00553.x},
  publisher = {Wiley}
}

@article{replica_MC,
  title     = {{Replica Monte Carlo Simulation of Spin-Glasses}},
  author    = {Swendsen, Robert H. and Wang, Jian-Sheng},
  journal   = {Physical Review Letters},
  volume    = {57},
  number    = {21},
  pages     = {2607--2609},
  year      = {1986},
  doi       = {10.1103/PhysRevLett.57.2607},
  url       = {https://link.aps.org/doi/10.1103/PhysRevLett.57.2607},
  publisher = {American Physical Society}
}

@book{book_variational_inference,
  author={Wainwright, Martin J. and Jordan, Michael I.},
  title={{Graphical Models, Exponential Families, and Variational Inference}},
  year={2008},
  series={Foundations and Trends® in Machine Learning},
  volume={1},
  issn = {1935-8237},
  number={},
  pages={},
  doi={10.1561/2200000001}
}

@inproceedings{
denoising_diffusion_samplers_vargas2023,
title={{Denoising Diffusion Samplers}},
author={Francisco Vargas and Will Grathwohl and Arnaud Doucet},
booktitle={{The Eleventh International Conference on Learning Representations }},
year={2023},
url={https://openreview.net/forum?id=8pvnfTAbu1f}
}

@article{
optimal_control_berner2024,
title={{An optimal control perspective on diffusion-based generative modeling}},
author={Julius Berner and Lorenz Richter and Karen Ullrich},
journal={Transactions on Machine Learning Research},
issn={2835-8856},
year={2024},
url={https://openreview.net/forum?id=oYIjw37pTP},
note={}
}

@inproceedings{
path_integral_sampler,
title={{Path Integral Sampler: A Stochastic Control Approach For Sampling}},
author={Qinsheng Zhang and Yongxin Chen},
booktitle={{International Conference on Learning Representations}},
year={2022},
url={https://openreview.net/forum?id=_uCb2ynRu7Y}
}

@article{brunick_shreve,
author = {Gerard Brunick and Steven Shreve},
title = {{Mimicking an Itô process by a solution of a stochastic differential equation}},
volume = {23},
journal = {The Annals of Applied Probability},
number = {4},
publisher = {Institute of Mathematical Statistics},
pages = {1584 -- 1628},
year = {2013},
doi = {10.1214/12-AAP881},
URL = {https://doi.org/10.1214/12-AAP881}
}

@article{time_reversal_diffusion,
author = {Patrick Cattiaux and Giovanni Conforti and Ivan Gentil and Christian L{\'e}onard},
title = {{Time reversal of diffusion processes under a finite entropy condition}},
volume = {59},
journal = {Annales de l'Institut Henri Poincaré, Probabilités et Statistiques},
number = {4},
publisher = {Institut Henri Poincaré},
pages = {1844 -- 1881},
year = {2023},
doi = {10.1214/22-AIHP1320},
URL = {https://doi.org/10.1214/22-AIHP1320}
}

@misc{SL_joint_applications_2025,
      title={{Joint stochastic localization and applications}}, 
      author={Tom Alberts and Yiming Xu and Qiang Ye},
      year={2025},
      eprint={2505.13410},
      archivePrefix={arXiv},
      url={https://arxiv.org/abs/2505.13410}, 
}

@book{robert2004monte,
  title     = {{Monte Carlo Statistical Methods}},
  author    = {Robert, Christian P. and Casella, George},
  series    = {Springer Texts in Statistics},
  edition   = {2nd},
  year      = {2004},
  publisher = {Springer-Verlag},
  address   = {New York},
  isbn      = {0-387-21239-6},
  doi       = {10.1007/978-1-4757-4145-2}
}

@book{kroese2011handbook,
  title     = {{Handbook of Monte Carlo Methods}},
  author    = {Kroese, Dirk P. and Taimre, Thomas and Botev, Zdravko I.},
  series    = {Wiley Series in Probability and Statistics},
  year      = {2011},
  publisher = {John Wiley \& Sons, Inc.},
  isbn      = {9780470177938},
  doi       = {10.1002/9781118014967}
}

@book{krauth2006statistical,
  title={{Statistical mechanics: algorithms and computations}},
  author={Krauth, Werner},
  volume={13},
  year={2006},
  publisher={OUP Oxford}
}

@book{stoltz2010free,
  title     = {{Free Energy Computations: A Mathematical Perspective}},
  author    = {Leli{\`e}vre, Tony and Rousset, Mathias and Stoltz, Gabriel},
  publisher = {World Scientific},
  year      = {2010},
  isbn      = {9781848162488}
}

@article{score_matching_article,
  author  = {Aapo Hyv{{\"a}}rinen},
  title   = {{Estimation of Non-Normalized Statistical Models by Score Matching}},
  journal = {Journal of Machine Learning Research},
  year    = {2005},
  volume  = {6},
  number  = {24},
  pages   = {695--709},
  url     = {http://jmlr.org/papers/v6/hyvarinen05a.html}
}

@inproceedings{grenioux2025improving,
  title = {{Improving the evaluation of samplers on multi-modal targets}},
  author = {Grenioux, Louis and Noble, Maxence and Gabri{\'e}, Marylou},
  year = {2025},
  booktitle = {{Frontiers in Probabilistic Inference: Learning meets Sampling}},
  url = {https://openreview.net/forum?id=d91E9RhVFU},
}

@inproceedings{eldan_pathwise,
      author    = {Ronen Eldan},
      title     = {{Analysis of high-dimensional distributions using pathwise methods}},
      booktitle = {{Proceedings of the International Congress of Mathematicians (ICM 2022)}},
      year      = {2022},
      pages     = {4246--4270},
      isbn      = {9783985470648},
      doi       = {10.4171/icm2022/61}
}

@inproceedings{Geyer1991MCMCML,
  author       = {Geyer, Charles J.},
  title        = {{Markov Chain Monte Carlo Maximum Likelihood}},
  booktitle    = {{Proceedings of the 23rd Symposium on the Interface between Computer Science and Statistics}},
  series       = {Interface ’91},
  year         = {1991},
  pages        = {156--163},
  publisher    = {Interface Foundation of North America}
}

@article{var_inf_mode_collapse,
author = {Soletskyi, Roman and Gabrie, Marylou and Loureiro, Bruno},
year = {2025},
month = {06},
pages = {},
title = {{A theoretical perspective on mode collapse in variational inference}},
volume = {6},
journal = {Machine Learning: Science and Technology},
doi = {10.1088/2632-2153/adde2a}
}

@InProceedings{var_inf_paper_1,
  title = 	 {{Variational refinement for importance sampling using the forward Kullback-Leibler divergence}},
  author =       {Jerfel, Ghassen and Wang, Serena and Wong-Fannjiang, Clara and Heller, Katherine A. and Ma, Yian and Jordan, Michael I.},
  booktitle = 	 {{Proceedings of the Thirty-Seventh Conference on Uncertainty in Artificial Intelligence}},
  pages = 	 {1819--1829},
  year = 	 {2021},
  volume = 	 {161},
  series = 	 {Proceedings of Machine Learning Research},
  publisher =    {PMLR},
  url = 	 {https://proceedings.mlr.press/v161/jerfel21a.html}
}

@InProceedings{var_inf_paper_2,
  title = 	 {{Beyond {ELBO}s: A Large-Scale Evaluation of Variational Methods for Sampling}},
  author =       {Blessing, Denis and Jia, Xiaogang and Esslinger, Johannes and Vargas, Francisco and Neumann, Gerhard},
  booktitle = 	 {{Proceedings of the 41st International Conference on Machine Learning}},
  pages = 	 {4205--4229},
  year = 	 {2024},
  volume = 	 {235},
  series = 	 {Proceedings of Machine Learning Research},
  month = 	 {July},
  publisher =    {PMLR},
  url = 	 {https://proceedings.mlr.press/v235/blessing24a.html}
}

@InProceedings{normalizing_flows_var_inf,
  title = 	 {{Variational Inference with Normalizing Flows}},
  author = 	 {Rezende, Danilo and Mohamed, Shakir},
  booktitle = 	 {{Proceedings of the 32nd International Conference on Machine Learning}},
  pages = 	 {1530--1538},
  year = 	 {2015},
  volume = 	 {37},
  series = 	 {Proceedings of Machine Learning Research},
  address = 	 {Lille, France},
  month = 	 {07--09 Jul},
  publisher =    {PMLR},
  url = 	 {https://proceedings.mlr.press/v37/rezende15.html}
}

@inproceedings{kingma2014auto,
  title     = {{Auto-Encoding Variational Bayes}},
  author    = {Kingma, Diederik P. and Welling, Max},
  booktitle = {{Proceedings of the 2nd International Conference on Learning Representations (ICLR)}},
  year      = {2014},
  url       = {https://arxiv.org/abs/1312.6114}
}

@InProceedings{var_autoencoders,
  title = 	 {{Stochastic Backpropagation and Approximate Inference in Deep Generative Models}},
  author = 	 {Rezende, Danilo Jimenez and Mohamed, Shakir and Wierstra, Daan},
  booktitle = 	 {{Proceedings of the 31st International Conference on Machine Learning}},
  pages = 	 {1278--1286},
  year = 	 {2014},
  volume = 	 {32},
  number =       {2},
  series = 	 {Proceedings of Machine Learning Research},
  address = 	 {Bejing, China},
  publisher =    {PMLR},
  url = 	 {https://proceedings.mlr.press/v32/rezende14.html}
}

@inproceedings{
he2024zerothorder,
title={{Zeroth-Order Sampling Methods for Non-Log-Concave Distributions: Alleviating Metastability by Denoising Diffusion}},
author={Ye He and Kevin Rojas and Molei Tao},
booktitle={{The Thirty-eighth Annual Conference on Neural Information Processing Systems}},
year={2024},
url={https://openreview.net/forum?id=X3Aljulsw5}
}

@inproceedings{noble_lrds,
 author = {Noble, Maxence and Grenioux, Louis and Gabri\'{e}, Marylou and Oliviero Durmus, Alain},
 booktitle = {{International Conference on Representation Learning}},
 pages = {47046--47100},
 title = {{Learned Reference-based Diffusion Sampler for multi-modal distributions}},
 volume = {2025},
 year = {2025}
}

@inproceedings{annealed_VI_1_wu2020stochastic,
  title={{Stochastic normalizing flows}},
  author={Wu, Hao and K{\"o}hler, Jonas and No{\'e}, Frank},
  booktitle={{Advances in Neural Information Processing Systems}},
  volume={33},
  pages={5933--5944},
  year={2020}
}

@inproceedings{annealed_VI_2_arbel2021annealed,
  title={{Annealed flow transport Monte Carlo}},
  author={Arbel, Michael and Matthews, Alex and Doucet, Arnaud},
  booktitle={{Proceedings of the 38th International Conference on Machine Learning}},
  pages={318--330},
  year={2021},
  organization={PMLR}
}

@inproceedings{annealed_VI_3_doucet2022score,
  title={{Score-based diffusion meets annealed importance sampling}},
  author={Doucet, Arnaud and Grathwohl, Will and Matthews, Alexander G. and Strathmann, Heiko},
  booktitle={{Advances in Neural Information Processing Systems}},
  volume={35},
  pages={21482--21494},
  year={2022}
}

@inproceedings{score_based_improved,
author = {Chen, Hongrui and Lee, Holden and Lu, Jianfeng},
title = {{Improved analysis of score-based generative modeling: user-friendly bounds under minimal smoothness assumptions}},
year = {2023},
publisher = {JMLR.org},
booktitle = {{Proceedings of the 40th International Conference on Machine Learning}},
articleno = {187},
numpages = {29},
location = {Honolulu, Hawaii, USA},
series = {ICML'23}
}

@misc{montanari2024posteriorsamplinghighdimension,
      title={{Posterior Sampling in High Dimension via Diffusion Processes}}, 
      author={Andrea Montanari and Yuchen Wu},
      year={2024},
      eprint={2304.11449},
      archivePrefix={arXiv preprint},
      url={https://arxiv.org/abs/2304.11449}, 
}

@inproceedings{turner2019metropolis,
  title={{Metropolis-Hastings generative adversarial networks}},
  author={Turner, Ryan and Hung, Jason and Frank, Eibe and Saatchi, Yarin and Yosinski, Jason},
  booktitle={{Proceedings of the 36th International Conference on Machine Learning}},
  volume={97},
  series={Proceedings of Machine Learning Research},
  pages={6345--6353},
  year={2019},
  publisher={PMLR},
  url={https://proceedings.mlr.press/v97/turner19a.html}
}

@inproceedings{grenioux2023balanced,
  title={{Balanced training of energy-based models with adaptive flow sampling}},
  author={Grenioux, Louis and Moulines, {\'E}ric and Gabri{\'e}, Marylou},
  booktitle={{ICML 2023 Workshop on Structured Probabilistic Inference \& Generative Modeling}},
  year={2023},
  url={https://openreview.net/forum?id=AwJ2NqxWlk}
}

@InProceedings{sohl-dickstein15,
  title = 	 {{Deep Unsupervised Learning using Nonequilibrium Thermodynamics}},
  author = 	 {Sohl-Dickstein, Jascha and Weiss, Eric and Maheswaranathan, Niru and Ganguli, Surya},
  booktitle = 	 {{Proceedings of the 32nd International Conference on Machine Learning}},
  pages = 	 {2256--2265},
  year = 	 {2015},
  editor = 	 {Bach, Francis and Blei, David},
  volume = 	 {37},
  series = 	 {Proceedings of Machine Learning Research},
  address = 	 {Lille, France},
  month = 	 {07--09 Jul},
  publisher =    {PMLR},
  url = 	 {https://proceedings.mlr.press/v37/sohl-dickstein15.html}
}

@inbook{song_ermon, 
author = {Song, Yang and Ermon, Stefano}, 
title = {{Generative modeling by estimating gradients of the data distribution}}, 
year = {2019}, 
publisher = {Curran Associates Inc.}, 
address = {Red Hook, NY, USA}
}

@inproceedings{denoising_diffusion,
author = {Ho, Jonathan and Jain, Ajay and Abbeel, Pieter}, 
title = {{Denoising diffusion probabilistic models}}, year = {2020}, 
isbn = {9781713829546}, 
publisher = {Curran Associates Inc.}, 
address = {Red Hook, NY, USA}, 
booktitle = {{Proceedings of the 34th International Conference on Neural Information Processing Systems}}, 
articleno = {574}, 
numpages = {12}, 
location = {Vancouver, BC, Canada}, 
series = {NIPS '20} }

@inproceedings{sgm_chen, 
author = {Chen, Hongrui and Lee, Holden and Lu, Jianfeng}, 
title = {{Improved analysis of score-based generative modeling: user-friendly bounds under minimal smoothness assumptions}}, 
year = {2023}, 
publisher = {JMLR.org}, 
booktitle = {{Proceedings of the 40th International Conference on Machine Learning}}, 
articleno = {187}, 
numpages = {29}, 
location = {Honolulu, Hawaii, USA}, 
series = {ICML'23} 
}

@inproceedings{sgm_lee, 
author = {Lee, Holden and Lu, Jianfeng and Tan, Yixin}, 
title = {{Convergence for score-based generative modeling with polynomial complexity}}, 
year = {2022}, 
isbn = {9781713871088}, 
publisher = {Curran Associates Inc.}, 
address = {Red Hook, NY, USA}, 
booktitle = {{Proceedings of the 36th International Conference on Neural Information Processing Systems}}, 
articleno = {1662}, 
numpages = {13}, 
location = {New Orleans, LA, USA}, 
series = {NIPS '22} 
}

@inproceedings{
sgm_lee_2,
title={{Convergence of score-based generative modeling for general data distributions}},
author={Holden Lee and Jianfeng Lu and Yixin Tan},
booktitle={{NeurIPS 2022 Workshop on Score-Based Methods}},
year={2022},
url={https://openreview.net/forum?id=Sg19A8mu8sv}
}

@inproceedings{
sgm_li,
title={{Towards Non-Asymptotic Convergence for Diffusion-Based Generative Models}},
author={Gen Li and Yuting Wei and Yuxin Chen and Yuejie Chi},
booktitle={{The Twelfth International Conference on Learning Representations}},
year={2024},
url={https://openreview.net/forum?id=4VGEeER6W9}
}

@article{Chen2021KLS,
  author    = {Chen, Yuansi},
  title     = {{An Almost Constant Lower Bound of the Isoperimetric Coefficient in the {KLS} Conjecture}},
  journal   = {Geometric and Functional Analysis},
  volume    = {31},
  pages     = {34--61},
  year      = {2021},
  doi       = {10.1007/s00039-021-00558-4},
  publisher = {Springer}
}

@article{ConfortiDurmusSilveri2025,
  author = {Giovanni Conforti and Alain Durmus and Marta Gentiloni Silveri},
  title = {{KL Convergence Guarantees for Score Diffusion Models under Minimal Data Assumptions}},
  journal = {SIAM Journal on Mathematics of Data Science},
  year = {2025},
  volume = {7},
  number = {1},
  pages = {86--109},
  month = {jan},
  doi = {10.1137/23M1613670},
  url = {https://epubs.siam.org/doi/10.1137/23M1613670}
}

\newpage
\appendix

\section*{Organization of the Appendix}
The Appendix is organized as follows.
In Appendix \ref{appendix_main_results}, we prove Theorem \ref{theorem_slips_tv_guarantee} and state and prove Corollary \ref{corollary_kL}. The proof of the optimality result for the log-SNR adapted discretization, Lemma \ref{lemma_opt_step_size}, is the content of Appendix \ref{appendix_discretization}.
In Appendix \ref{appendix_SDE_theory}, we discuss for completeness the SDE description of the general observation process in detail, complementing the discussion in \cite{SLIPS}.

\section{Proofs of the main results}\label{appendix_main_results}
This part of the Appendix is subdivided in three sections. In Appendix \ref{section_sub_appendix_main}, we prove the main results, Theorem \ref{theorem_slips_tv_guarantee} and Corollary \ref{corollary_kL}. In Appendices \ref{section_sub_appendix_main_sup1} and \ref{section_sub_appendix_main_sup2} we prove supplementary Lemmas used in the proofs of Appendix \ref{section_sub_appendix_main}, namely Lemmas \ref{lemma_tv_approx_SL} and \ref{lemma_discretization_error}.

\subsection{Proofs of Theorem \ref{theorem_slips_tv_guarantee} and Corollary \ref{corollary_kL}}\label{section_sub_appendix_main}
We start with the proof of Theorem \ref{theorem_slips_tv_guarantee}.
\begin{proof}[Proof of Theorem \ref{theorem_slips_tv_guarantee}]
    First, observe that by the triangle inequality and the scale invariance of the TV distance, we have
    \begin{align}\label{eq_tv_dist_first_bound} 
        d_{\textnormal{TV}}(\tilde{\pi}_{T},\pi)
        \leq &  d_{\textnormal{TV}}(\tilde{\pi}_{T},\pi_{T})  
        +d_{\textnormal{TV}}(\pi_{T},\pi)
        \\ \notag = & d_{\textnormal{TV}}(\tilde{p}_{T},p_{T})  
        +d_{\textnormal{TV}}(\pi_{T},\pi),
    \end{align}
    where $\pi_T = (1/T)_\# p_T$.
    Plugging in the estimate from Lemma \ref{lemma_tv_approx_SL} 
    for the first quantity on the above's right hand side and the estimate from Lemma \ref{lemma_discretization_error} for the second quantity, we find 
    \begin{align}
        d_{\textnormal{TV}}(\tilde{\pi}_{T},\pi)\leq &
        d_{\textnormal{TV}}(\tilde{p}_{t_0},p_{t_0})+
        \sqrt{d\cdot C_{\textnormal{disc}}}
        +\sqrt{\frac{1}{\sigma^2} T \varepsilon_0^2}
        +
        \frac{1}{2}
        \|\nabla \log \pi\|_{L^2(\pi)} \,\sqrt{d\frac{\sigma^2}{T}},
    \end{align}
    which is the claim of the Theorem.
\end{proof}
We next state a formal version of Corollary \ref{corollary_kL}.
\begin{corollary}\label{corollary_kL_appendix}
    Let $\varepsilon>0$. In the setup of Theorem \ref{theorem_slips_tv_guarantee}, assume $R^2_{\pi},\|\nabla \log \pi \|_{L^2(\pi)}^2 \in \mc{O}(d)$.
    Consider SLIPS run with the log-SNR adapted discretization $(t_k)_{0\leq k \leq K}=\big((\frac{t_K}{t_0})^{\frac{k}{K}}t_0 \big)_{0\leq k \leq K}$.
    \vspace{2mm}\\Then,
    we have $d_{\textnormal{TV}}(\tilde{\pi}_{T},\pi)\leq\varepsilon$, for $T=C_0\frac{1}{\varepsilon^2}\|\nabla \log \pi \|_{L^2(\pi)}^2 \,R^2_{\pi}$ and $K,\,\varepsilon_0$, and $d_{\textnormal{TV}}(\tilde{p}_{t_0},p_{t_0})$ fulfilling 
    \begin{align}\label{eq_choices_tv_corollary_1}
        K\geq & C_1 \,\frac{d}{\varepsilon^2}\log^2 \Big( \frac{d^2}{t_0 \varepsilon^2}\Big),
        \\ \label{eq_choices_tv_corollary_2}
        \varepsilon_0\leq &C_2\,\frac{\varepsilon^2}{d},
        \\ \label{eq_choices_tv_corollary_3} 
        d_{\textnormal{TV}}(\tilde{p}_{t_0},p_{t_0}) \leq 
        & \frac{1}{4} \,\varepsilon,
    \end{align}
    where the constants $C_0,C_1,C_2>0$ depend only on the constants hidden by the Landau conditions on $R^2_{\pi},\|\nabla \log \pi \|_{L^2(\pi)}^2$.
\end{corollary}
\begin{proof}[Proof of Corollary \ref{corollary_kL}/\ref{corollary_kL_appendix}]
    Recall from Remark \ref{remark_c_disc}, that for the log-SNR adapted discretization we have $C_{\textnormal{disc}}\leq \kappa (\kappa K +1)$, where $\kappa =e^{\frac{1}{K}\log \frac{T}{t_0}}-1$, and that the default choice of $\sigma^2$ is $\sigma^2=R^2_{\pi}/d$.
    Hence, the bound of Theorem \ref{theorem_slips_tv_guarantee} becomes
    \begin{align}
        d_{\textnormal{TV}}(\tilde{\pi}_{T},\pi)\leq &
        d_{\textnormal{TV}}(\tilde{p}_{t_0},p_{t_0})+
        \sqrt{d\cdot \kappa (\kappa K +1)}
        +\sqrt{\frac{d}{R^2_{\pi}} T \varepsilon_0^2}
        +\frac{1}{2}\|\nabla \log \pi(X)\|_{L^2} \,\sqrt{\frac{R^2_{\pi}}{T}}.
    \end{align}
    Since $\kappa\sim \frac{1}{K} \log \frac{T}{t_0}$ if $\frac{1}{K} \log \frac{T}{t_0}\ll 1$, it follows that $T,K,\varepsilon_0$, and $d_{\textnormal{TV}}(\tilde{p}_{t_0},p_{t_0})$ of order as proposed in the Corollary ensure that 
    $d_{\textnormal{TV}}(\tilde{\pi}_{T},\pi)\leq\varepsilon$.
\end{proof}

\subsection{Proof of Lemma \ref{lemma_tv_approx_SL}}\label{section_sub_appendix_main_sup1}
We give the proof of Lemma \ref{lemma_tv_approx_SL}, which establishes a bound on the information error.
\begin{proof}[Proof of Lemma \ref{lemma_tv_approx_SL}]
    Since $\pi,\,\pi_{t}$ are absolutely continuous with respect to the Lebesgue measure, it holds that
    \begin{align}
        d_{\textnormal{TV}}(\pi,\pi_{t})
            = &
            \frac{1}{2}
            \int_{\mb{R}^d}|\pi(x)-\pi_{t}(x)|\,dx,
    \end{align}
    see \cite{eberle2024mcmc}, Lemma 3.2.
    Using that $\pi_{t} = \pi \ast \mc{N}(0,\frac{\sigma^2}{t})$ and that for any probability density $\rho$ trivially $\pi(x)=\int_{\mb{R}^d}\pi(x)\rho(dz)$, we thus find that
    \begin{align}
        2\cdot d_{\textnormal{TV}}(\pi,\pi_{t})
        = &
        \int_{\mb{R}^d}|\int_{\mb{R}^d} (\pi(x)-\pi(x-y))\, \mc{N}(y,0,\frac{\sigma^2}{t})\,dy|\,dx
        \\ \notag 
        = &
        \int_{\mb{R}^d}|\int_{\mb{R}^d} \int_0^1 \nabla\pi(x-sy)^\top y\,ds\,\mc{N}(y,0,\frac{\sigma^2}{t})\,dy|\,dx,
    \end{align}
    where the second equality follows by the fundamental theorem of calculus.
    Pulling the absolute value inside the integral and applying Cauchy Schwarz, we can bound the TV distance by
    \begin{align}
        2\cdot 
        d_{\textnormal{TV}}(\pi,\pi_{t})
        \leq &
        \int_{\mb{R}^d}\int_{\mb{R}^d} \int_0^1 \|\nabla\pi(x-sy)\|\cdot\|y\|\,ds\,\mc{N}(y,0,\frac{\sigma^2}{t})\,dy\,dx.
    \end{align}
    Applying Fubini and the change of variables $z=x-sy$, we obtain 
    \begin{align}\notag
        2\cdot 
        d_{\textnormal{TV}}(\pi,\pi_{t})
        \leq &
        \int_0^1
        \int_{\mb{R}^d}\|y\|
        \Big(\int_{\mb{R}^d}  \|\nabla\pi(x-sy)\|\,dx\Big)\,\mc{N}(y,0,\frac{\sigma^2}{t})\,dy\,ds
        \\ \notag = &
        \int_0^1
        \int_{\mb{R}^d}\|y\|
        \Big(\int_{\mb{R}^d}  \|\nabla\pi(z)\|\,dz\Big)\,\mc{N}(y,0,\frac{\sigma^2}{t})\,dy\,ds
        \\ \notag =& 
        \Big(\int_{\mb{R}^d}  \|\nabla\pi(z)\|\,dz\Big)\cdot
        \mb{E}_{y\sim \mc{N}(0,\frac{\sigma^2}{t})}\big[\|y\|\big].
    \end{align}
    Using that $\nabla\pi(x)=\pi(x)\nabla\log \pi(x)$ and that the map $\mb{R}\ni a\mapsto a^2$ is convex, we find by Jensen and the above that for $X\sim \pi$ and $y\sim \mc{N}(0,\frac{\sigma^2}{t})$ we have
    \begin{align}\notag 
        d_{\textnormal{TV}}(\pi,\pi_{t})
        \leq&
        \frac{1}{2}\|\nabla \log \pi(X)\|_{L^2} \cdot\|y\|_{L^2}
        =
        \frac{1}{2}
        \|\nabla \log \pi\|_{L^2(\pi)} \,\sqrt{d\frac{\sigma^2}{t}},
    \end{align}
    which is the claim.
\end{proof}

\subsection{Proof of Lemma \ref{lemma_discretization_error}}\label{section_sub_appendix_main_sup2}
We now discuss how to establish the bound on the discretization error, which is the content of Lemma \ref{lemma_discretization_error}.
Instead of comparing SLIPS with the SL observation process directly, we want to study the discretization error by comparing SLIPS with \textit{exact} initialization $p_{t_0}$ with the SL observation process. This is possible due to the following Lemma.
\begin{lemma}\label{lemma_triangle_exact}
    Denote by $\tilde{p}^{\textnormal{ex}}_{t_k}$ the law of the $k$-th SLIPS iterate if SLIPS is initialized exactly, e.g.\ $\widetilde{Y}_{t_0}^{\textnormal{ex}}\sim p_{t_0}$.
    Then, it holds that
    \begin{align}
        d_{\textnormal{TV}}(\tilde{p}_{t_K},p_{t_K})
        \leq
        d_{\textnormal{TV}}(\tilde{p}_{t_0},p_{t_0})  
        +d_{\textnormal{TV}}(\tilde{p}^{\textnormal{ex}}_{t_K},p_{t_K}).
    \end{align}
\end{lemma}
\begin{proof}
    By the triangle inequality it holds that 
    \begin{align}
        d_{\textnormal{TV}}(\tilde{p}_{t_K},p_{t_K})
        \leq
        d_{\textnormal{TV}}(\tilde{p}_{t_K},\tilde{p}^{\textnormal{ex}}_{t_K})  
        +d_{\textnormal{TV}}(\tilde{p}^{\textnormal{ex}}_{t_K},p_{t_K}).
    \end{align}
    Since both $\tilde{p}_{t_k}$ and $\tilde{p}^{\textnormal{ex}}_{t_k}$ are obtained through the same recursion, we have  
    \begin{align}
        d_{\textnormal{TV}}(\tilde{p}_{t_K},\tilde{p}^{\textnormal{ex}}_{t_K})  
        \leq d_{\textnormal{TV}}(\tilde{p}_{t_0},\tilde{p}^{\textnormal{ex}}_{t_0})  
        =d_{\textnormal{TV}}(\tilde{p}_{t_0},p_{t_0}),
    \end{align}
    which yields the result.
\end{proof}
The remaining part of this Section is devoted to bounding $d_{\textnormal{TV}}(\tilde{p}^{\textnormal{ex}}_{t_K},p_{t_K})$. We do so following the approach applied in \cite{Sampling_is_as_easy_as_learning_the_score_chen2023sampling}, \cite{benton2024nearly} for the setting of generative modeling with the reverse Ornstein Uhlenbeck process. 
In Section \ref{section_sub_proof_strategy} we already explained that the EM discretization used in SLIPS can be understood as a continuous time process that solves on $[t_k,t_{k+1})$ the SDE with piecewise constant drift
\begin{align}\label{eq_dynamics_SLIPS}
    d\widetilde{Y}_t = \hat{u}_{t_k}(\widetilde{Y}_{t_k})\,dt+ \sigma \,d\bar{B}_t.
\end{align}
Our interest in this continuous interpretation stems from the fact
that it allows us to compare the output of the (discrete) SLIPS algorithm with the normalized (continuous) SL observation process by comparing measures on path space, instead of considering only the marginal laws $\tilde{p}^{\textnormal{ex}}_{t_K}$ and $p_{t_K}$. As mentioned already, the benefit is that, switching from TV to KL divergence by Pinsker's inequality, comparing laws of solutions 
to SDEs becomes tractable using tools such as Girsanov's Theorem.
\vspace{2mm}\\In the following, we write
\begin{align}\label{eq_laws_pathwise}
    \tilde{P}^{\textnormal{ex}}\coloneqq \textnormal{Law}\big(
    (\widetilde{Y}^{\textnormal{ex}}_{t})_{t\in [t_0,t_K)}
    \big)\,\,\,\,\,\,\,\,\textnormal{and}\,\,\,\,\,\,\,\,P \coloneqq \textnormal{Law}\big((Y_{t})_{t\in [t_0,t_K)}\big),
\end{align}
where we denote by $\widetilde{Y}^{\textnormal{ex}}$ the solution of \eqref{eq_dynamics_SLIPS} with exact initial condition $\widetilde{Y}^{\textnormal{ex}}_{t_0}\sim p_{t_0}$.
\vspace{2mm}\\The next Lemma contains our bound on $d_{\textnormal{TV}}(\tilde{p}^{\textnormal{ex}}_{t_K},p_{t_K})$ obtained by the path measure approach.
\begin{lemma}\label{lemma_KL_discretization_error}
    Assume Assumptions \ref{assumption_moment} and  \ref{assumption_posterior_estimator_SL_process} and consider for $K\in \mb{N}$ an arbitrary discretization $0<t_0<t_1<\dots<t_K$. Then, for SLIPS with exact initialization it holds that
    \begin{align}
        d_{\textnormal{TV}}(\tilde{p}^{\textnormal{ex}}_{t_K},p_{t_K}) 
        \leq 
        \sqrt{\frac{1}{\sigma^2} t_K \varepsilon_0^2
        +
        d \cdot C_{\textnormal{disc}}},
    \end{align}
    where we recall that $C_{\textnormal{disc}}$, defined in \eqref{eq_sum_discretization}, is given by
    \begin{align}\label{eq_sum_discretization_2}
        C_{\textnormal{disc}} &=
        \sum_{1\leq k\leq K-1} \max\big\{0,(t_{k+1}-t_k)-(t_{k}-t_{k-1})\big\}
        \frac{1}{t_k}
        + \frac{t_1-t_0}{t_0}.
    \end{align}
\end{lemma}
Observe that the proof of Lemma \ref{lemma_discretization_error} now follows directly.
\begin{proof}[Proof of Lemma \ref{lemma_discretization_error}.]
    Plugging the bound of Lemma \ref{lemma_KL_discretization_error} into the bound of Lemma \ref{lemma_triangle_exact} and using that $\sqrt{a+b}\leq \sqrt{a}+\sqrt{b}$ for $a,b\geq 0$, we obtain the bound of Lemma \ref{lemma_discretization_error}, namely
    \begin{align}
        d_{\textnormal{TV}}(\tilde{p}_{t_K},p_{t_K}) 
        \leq &
        d_{\textnormal{TV}}(\tilde{p}_{t_0},p_{t_0})+
        \sqrt{\frac{1}{\sigma^2} t_K \varepsilon_0^2}+
        \sqrt{d\cdot C_{\textnormal{disc}}}
        .
    \end{align}
\end{proof}
For the proof of Lemma \ref{lemma_KL_discretization_error}, we need the following two supplementary results. The first computes a formula for the change of $u_t(Y_t)$ in $L^2$, exploiting that this process is a martingale. The second bounds a quantity appearing in the TV bound of the discretization error. The first reads as follows.
\begin{lemma}\label{lemma_u_L2}
    Under Assumption \ref{assumption_moment},
    the process $(u_t(Y_t))_{t\geq 0}$ is an $\mc{F}^Y$-martingale and for $0\leq s \leq t$ we have
    \begin{align}
        \label{eq_derivative_discretization_error}
        &
        \big\|
        u_{t}(Y_{t})- u_s (Y_s)
        \big\|^2_{L^2}
        =
        \mb{E}\,\big[
        \textnormal{Tr}\Cov (X|Y_s)
        \big]
        -\mb{E}\,\big[
        \textnormal{Tr} 
        \Cov (X|Y_t)
        \big].    
        \end{align}
\end{lemma}
\begin{proof}[Proof of Lemma \ref{lemma_u_L2}]
    The fact that $(u_t(Y_t))_{t\geq 0}$ is a martingale is well known in the literature and we briefly recall the argument. Since the process $(u_t(Y_t))_{t\geq 0}$ solves the SDE
    \begin{align}\label{eq_posterior_expectation_sde}
        du_t(Y_t) = \frac{1}{\sigma} \Cov(X|Y_t)\,d\bar{B}_t,
    \end{align}
    see Section 4.1 in \cite{montanari2023samplingdiffusionsstochasticlocalization}, it is a local martingale. Since $\pi \in \mc{P}_2(\mb{R}^d)$, by Jensen $\mb{E}\big[\|u_t(Y_t)\|^2\big]\leq\mb{E}\big[\|X\|^2\big]<\infty$, so the process is a true martingale.    
    \vspace{2mm}\\To establish \eqref{eq_derivative_discretization_error}, we now exploit that $u_t(Y_t)=\mb{E}[X|Y_t]$ is a martingale.
    By the tower property of conditional expectation, we have
    \begin{align}\notag
        &\big\|
        u_{t}(Y_{t})- u_s (Y_s)
        \big\|^2_{L^2}
        = \|u_t(Y_t)\|^2+\|u_s(Y_s)\|^2-2\mb{E}[\langle u_t(Y_t),u_s(Y_s)\rangle]
        \\ \notag =& 
        \|u_t(Y_t)\|^2-\|u_s(Y_s)\|^2.
    \end{align}
    Since $q_t(x|y)p_t(y)=p_t(y|x)\pi(x)$, we have 
    $\mb{E}[\int_{\mb{R}^d}xx^\top\,q_t(dx|Y_t)]=\mb{E}[XX^\top]$ for all $t\geq  0$. Therefore, we find that
    \begin{align}\notag
        &\|u_t(Y_t)\|^2-\|u_s(Y_s)\|^2 
        \\ \notag =& 
        \big(\mb{E}[\int_{\mb{R}^d}xx^\top\,q_s(dx|Y_s)]
        -\|u_s(Y_s)\|^2\big)
        -\big(\mb{E}[\int_{\mb{R}^d}xx^\top\,q_t(dx|Y_t)]-\|u_t(Y_t)\|^2\big)
        \\ \notag 
        =&
        \mb{E}\,\big[
        \textnormal{Tr}\Cov (X|Y_s)
        \big]
        -\mb{E}\,\big[
        \textnormal{Tr} 
        \Cov (X|Y_t)
        \big],
    \end{align}
    which yields the claim.
\end{proof}
\begin{remark}\label{remark_sum_bound}
    One could have proven the identity \eqref{eq_derivative_discretization_error} in Lemma \ref{lemma_u_L2} also by combining the SDE \eqref{eq_posterior_expectation_sde} for the drift $u_t(Y_t)$ with It\^o's isometry and equation $(11)$ from \cite{Eldan2020}, which reads
    \begin{align}
        \partial_t\textnormal{Tr} \,\mb{E}[\Cov(X|Y_t)] =- \frac{1}{\sigma^2}\textnormal{Tr} \,\mb{E}[\Cov^2(X|Y_t)].
    \end{align}
    This SDE-based approach is used in the theoretical study of generative modeling with the Ornstein Uhlenbeck process, where the drift of the diffusion is no martingale. The key that allows for our simpler approach is that we know \textit{a priori} that $u_t(Y_t)$ is a martingale. This approach can also be applied to study the discretization error of SLIPS with non-linear denoising schedule, but the SDE has in that case a much more complicated form than \eqref{eq_posterior_expectation_sde}. In particular, the drift does not vanish.
\end{remark}
The next result establishes a bound on the sum \eqref{eq_sum_disc_error}, which arises in the proof of Lemma \ref{lemma_KL_discretization_error}. The sum-reordering technique applied in the proof is due to \cite{benton2024nearly}.
\begin{lemma}\label{lemma_E_bound}
    Assume Assumption \ref{assumption_moment}. Defining for $t_0\leq t_k \leq t$ the quantity $E_{t_k,t}\coloneqq \big\|
        u_{t}(Y_{t})- u_{t_k}(Y_{t_k})
        \big\|^2_{L^2}$, it holds that
    \begin{align}\label{eq_corollary_disc_error_2}
        \sum_{0\leq k\leq K-1} \int_{t_k}^{t_{k+1}}
        E_{t_k,t}\,dt 
        \leq d\sigma^2 \cdot C_{\textnormal{disc}},
    \end{align}
    where $C_{\textnormal{disc}}$ is defined in \eqref{eq_sum_discretization}.
\end{lemma}
\begin{proof}[Proof of Lemma \ref{lemma_E_bound}]
    Due to Lemma \ref{lemma_u_L2}, we have  
    $E_{t_k,t}= \mb{E}\,\big[
        \textnormal{Tr}\Cov (X|Y_{t_k})
        \big]
        -\mb{E}\,\big[
        \textnormal{Tr} 
        \Cov (X|Y_{t})
        \big]$.
    Since Lemma \ref{lemma_u_L2} implies further that
    $t\mapsto \mb{E}\,\big[\textnormal{Tr}\Cov (X|Y_{t})\big]$ is decreasing, we have
    \begin{align}\notag
        \sum_{0\leq k\leq K-1} \int_{t_k}^{t_{k+1}}
        E_{t_k,t}\,dt
        \leq &
        \sum_{0\leq k\leq K-1} 
        \int_{t_k}^{t_{k+1}}
        \mb{E}\,\big[
        \textnormal{Tr}\Cov (X|Y_{t_k})
        \big]
        -\mb{E}\,\big[
        \textnormal{Tr} 
        \Cov (X|Y_{t_{k+1}})
        \big]
        \,dt.
    \end{align}
    Observe that we can rewrite the right hand side as 
    \begin{align}\notag
        &\sum_{0\leq k\leq K-1} 
        \int_{t_k}^{t_{k+1}}
        \mb{E}\,\big[
        \textnormal{Tr}\Cov (X|Y_{t_k})
        \big]
        -\mb{E}\,\big[
        \textnormal{Tr} 
        \Cov (X|Y_{t_{k+1}})
        \big]
        \,dt
        \\ \notag = &\sum_{0\leq k\leq K-1} (t_{k+1}-t_k)\big(
        \mb{E}\,\big[
        \textnormal{Tr}\Cov (X|Y_{t_k})
        \big]
        -\mb{E}\,\big[
        \textnormal{Tr} 
        \Cov (X|Y_{t_{k+1}})
        \big]
        \big)
        \\ \notag =&
        \sum_{1\leq k\leq K-1} \big((t_{k+1}-t_k)-(t_{k}-t_{k-1})\big)
        \mb{E}\,\big[
        \textnormal{Tr}\Cov (X|Y_{t_k})
        \big]
        \\ \notag & + (t_1-t_{0})
        \mb{E}\,\big[
        \textnormal{Tr}\Cov (X|Y_{t_0})
        \big]
        -(t_{K}-t_{K-1})\mb{E}\,\big[
        \textnormal{Tr}\Cov (X|Y_{t_K})
        \big].
    \end{align}
    We now bound the conditional covariances.
    Note that since $Y_t=tX+\sigma B_t$, we have $X=\frac{1}{t}Y_t-\frac{\sigma}{t}B_t$. Therefore, we find
    \begin{align}\label{eq_bound_e_trace}
        \mb{E}[\textnormal{Tr}\Cov(X|Y_t)] =
        \mb{E}[\textnormal{Tr}\Cov\big(\frac{\sigma}{t}B_t|Y_t\big)]
        \leq d\frac{\sigma^2}{t},
    \end{align}
    where in the inequality we used that 
    \begin{align}
        \mb{E}[\textnormal{Tr}\Cov\big(B_t|Y_t\big)] \leq 
        \mb{E}[\textnormal{Tr}(B_tB_t^\top)]=dt.
    \end{align}
    Hence, plugging \eqref{eq_bound_e_trace} into the previous identities, we obtain
    \begin{align}\label{eq_integral_minimization}
        \sum_{0\leq k\leq K-1} \int_{t_k}^{t_{k+1}}
        E_{t_k,t}\,dt
        \leq &
        d\sigma^2\Big(\sum_{1\leq k\leq K-1} \max\big\{0,\big((t_{k+1}-t_k)-(t_{k}-t_{k-1})\big)\big\}
        \frac{1}{t_k}\Big)
        + d\sigma^2 \frac{t_1-t_0}{t_0}.
    \end{align}
    Recalling the Definition of $C_{\textnormal{disc}}$ from \eqref{eq_sum_discretization} (or \eqref{eq_sum_discretization_2}), the above is precisely \eqref{eq_corollary_disc_error_2}.
\end{proof}
We now close this Section with the proof of Lemma \ref{lemma_KL_discretization_error}.
\begin{proof}[Proof of Lemma \ref{lemma_KL_discretization_error}.]
    Applying the data processing inequality and Pinsker's inequality, we find
    \begin{align}\label{eq_data_pinsker}
        d_{\textnormal{TV}}(\tilde{p}^{\textnormal{ex}}_{t_K},p_{t_K})
        \leq
        d_{\textnormal{TV}}(\tilde{P}^\textnormal{ex},P)
        \leq
        \sqrt{\frac{1}{2}\textnormal{KL}(P\|\tilde{P}^\textnormal{ex})},
    \end{align}
    where we recall that $P,\tilde{P}^\textnormal{ex}$ were defined in \eqref{eq_laws_pathwise}.
    The rest of the proof consists in bounding the KL divergence on the above's right hand side.
   Since the SL observation process $Y$ solves on $[0,\infty)$ the SDE \eqref{eq_SDE_Y}, it solves on $[t_0,\infty)$ the SDE
    \begin{align}\label{eq_SL_SDE_TV}
        \begin{cases}
            dY_t &= u_t(Y_t)\,dt+\sigma\,d\bar{B}_t,
            \\Y_{t_0}&\sim p_{t_0},
        \end{cases}
    \end{align}
    where $(\bar{B}_t)_{t\geq t_0}$ is a BM starting at time $t_0$ in zero that is independent of $Y_{t_0}$. Note that this is simply the BM from \eqref{eq_SDE_Y}, of which the $t_0$-th increment is subtracted.
    Since $(\widetilde{Y}^{\textnormal{ex}}_{t})_{t_0\leq t\leq t_K}$ solves \eqref{eq_dynamics_SLIPS}, both SDEs have the same constant diffusion coefficient. Note also that $\textnormal{Law}(\widetilde{Y}^{\textnormal{ex}}_{t_0})=\textnormal{Law}(Y_{t_0})=p_{t_0}$.
    We want to apply Girsanov's Theorem to compute the density of $\tilde{P}^{\textnormal{ex}}$ with respect to $P$, see \cite{karatzas1991brownian}, Theorem 5.1.  
    Our exponential martingale of choice is thus defined for $t_0\leq t \leq t_K$ by
    \begin{align}\notag
        \mc{E}_{t}\coloneqq 
        \exp\Big(&\frac{1}{\sigma}\sum_{0\leq k\leq K-1} \int_{t \wedge t_k}^{t \wedge t_{k+1}}
        \hat{u}_{t_k}(Y_{t_k})
        -u_s(Y_s)\,d\bar{B}_s
        \\ \notag &-\frac{1}{2\sigma^2}
        \sum_{0\leq k\leq K-1} \int_{t \wedge t_k}^{t \wedge t_{k+1}}
        \|\hat{u} _{t_k}(Y_{t_k})
        -u_s(Y_s)\|^2\,ds\Big).
    \end{align}
    Observe that $\hat{u}_{t_k}$ depends not only on $Y_{t_k}$, but also on the random variables used to run MALA and the previous iterates $Y_{t_l}$ for $l\leq k$ due to the initialization of MALA, see the discussion at the end of Section \ref{section_sub_slips_standard}. We denote the noise used to
    compute $\hat{u}_{t_k}$ by $Z^k$, so in particular $\hat{u}_{t_k}(Y_{t_k})=\hat{u}_{t_k}(\{Y_{t_l}\}_{l\leq k},\{Z^l\}_{l\leq k})$.
    The random variables $Z^k$ are taken independently of $\bar{B}$, $Y$, independent across $0\leq k \leq K-1$, and we stick to the previous convention of making them implicit in the notation whenever possible. 
    Observe also that $\mc{E}_t$ is a measurable function of $(Y_s)_{t_0\leq s \leq t}$ and the noise involved in the posterior expectation estimation only, since $\bar{B}_t=\frac{1}{\sigma}(Y_t-Y_{t_0}-\int_{t_0}^t u_s(Y_s)ds)$ by \eqref{eq_SL_SDE_TV}.    
    \vspace{2mm}\\Denote by $\mb{P}$ the probability measure of the filtered probability space on which all the previous random variables are defined and by $\mb{E}$ the $\mb{P}$-expectation. 
    If the condition of Girsanov's Theorem is fulfilled, e.g.\ $\mc{E}_t$ is a martingale, under the probability measure $d\tilde{\mb{P}}\coloneqq \mc{E}_{t_K}d\mb{P}$, the process
    \begin{align}\notag
        (\tilde{B}_t)_{t_0 \leq t \leq t_K}\coloneqq \big(\bar{B}_t-
        \frac{1}{\sigma}\sum_{0\leq k\leq K-1} \int_{t \wedge t_k}^{t \wedge t_{k+1}}
        \hat{u}_{t_k}(Y_{t_k})
        -u_s(Y_s)\,ds\big)_{t_0 \leq t \leq t_K}
    \end{align}
    is a BM and by another Girsanov application one finds $\textnormal{Law}_{\tilde{\mb{P}}}(\tilde{B},\,Y_{t_0},Z)=\textnormal{Law}_{\mb{P}}(\bar{B},\,Y_{t_0},Z)$, where we wrote $Z=(Z^k)_{0\leq k \leq K-1}$. 
    We further write $\nu^{\textnormal{noise}}\coloneqq \textnormal{Law}(Z)$.
    Writing the equation \eqref{eq_SL_SDE_TV} in terms of $\tilde{B}$, we find that on $[t_k,t_{k+1})$ the process $Y$ solves
    \begin{align}\notag
        dY_t = \hat{u}_{t_k}(Y_{t_k})\,dt + \sigma \,d\tilde{B}_t. 
    \end{align}
    Under $\tilde{\mb{P}}$, this is the SDE description of SLIPS, so the process $Y$ has law $\tilde{P}^\textnormal{ex}$. We thus have for $A\in \mc{B}\big(\mc{C}([t_0,t_K),\mb{R}^d)\big)$ that 
    \begin{align}\notag
        \tilde{P}^\textnormal{ex}(A)= \mb{E}[\mc{E}_{t_K} \cdot \textbf{1}_{Y\in A}] 
        = \mb{E}_{\omega \sim P} \big[
        \mb{E}_{Z\sim \nu^{\textnormal{noise}}}
        [\mc{E}_{t_K}(\omega,Z)
        ]
        \cdot \textbf{1}_{\omega\in A}
        \big],
    \end{align}
    where we made explicit the dependence of $\mc{E}_{t_K}$ on $\omega \sim P$ and the noise $Z\sim \nu_{\textnormal{noise}}$ used to run the $K$ MALA chains for the $K$ posterior expectation estimations.
    This implies that a Radon Nikodym derivative of
    $\tilde{P}^\textnormal{ex}$ with respect to $P$ is given by 
    \begin{align}
        \frac{d\tilde{P}^\textnormal{ex}}{dP} (\cdot)
        \coloneqq\mb{E}_{Z\sim \nu^{\textnormal{noise}}}[\mc{E}_{t_K}(\cdot,Z)]>0.
    \end{align} 
    Hence, since $\frac{dP}{d\tilde{P}^\textnormal{ex}}=\big(\frac{d\tilde{P}^\textnormal{ex}}{dP}\big)^{-1}$, we have
    \begin{align}\label{eq_kl_identity_exact}
        \textnormal{KL}(P\|\tilde{P}^\textnormal{ex})
        =&
        \mb{E}_{Y\sim P}\,\big[\log
        \big(\mb{E}_{Z\sim \nu^{\textnormal{noise}}}[\mc{E}_{t_K}(Y,Z)]^{-1}\big)
        \big]
        \\ \notag \leq &
        \mb{E}_{Y\sim P}\,\big[\mb{E}_{Z\sim \nu^{\textnormal{noise}}}[\log
        \mc{E}_{t_K}(Y,Z)^{-1}]
        \big]
        \\ \notag 
        =&
        -\mb{E}\,\Big[
        \frac{1}{\sigma}\sum_{0\leq k\leq K-1} \int_{t_k}^{t_{k+1}}
        \hat{u}_{t_k}(Y_{t_k})
        -u_t(Y_t)\,dB_t
        \\ \notag &\quad -\frac{1}{2\sigma^2}
        \sum_{0\leq k\leq K-1} \int_{t_k}^{t_{k+1}}
        \|\hat{u} _{t_k}(Y_{t_k})
        -u_t(Y_t)\|^2\,dt
        \Big],
    \end{align}
    where the inequality follows by Jensen, since $x\mapsto \log(x^{-1})=-\log(x)$ is convex.
    Due to Assumptions \ref{assumption_moment} and \ref{assumption_posterior_estimator_SL_process}, the stochastic integral in the above is a martingale. Hence, we find
    \begin{align}\label{eq_Girsanov_naive}
        \textnormal{KL}(P\|\tilde{P}^\textnormal{ex})
        \leq &
        \mb{E}\,\Big[
        \frac{1}{2\sigma^2}
        \sum_{0\leq k\leq K-1} \int_{t_k}^{t_{k+1}}
        \|\hat{u} _{t_k}(Y_{t_k})
        -u_t(Y_t)\|^2\,dt
        \Big].
    \end{align}
    The above inequality relies on the fact that the application of Girsanov is feasible.
    Classical sufficient conditions ensuring applicability of   Girsanov, such as Novikov, do not hold provably in our setting. Though, it has been verified in \cite{Sampling_is_as_easy_as_learning_the_score_chen2023sampling}, in the proof of Theorem 9 therein, that by an approximation argument via stopping times, the inequality \eqref{eq_kl_identity_exact} still holds, and thus also \eqref{eq_Girsanov_naive}.
    While in \cite{Sampling_is_as_easy_as_learning_the_score_chen2023sampling} they consider the setup of the reverse Ornstein Uhlenbeck process, their approach translates 
    immediately, so we refer to  \cite{Sampling_is_as_easy_as_learning_the_score_chen2023sampling} for the details.
    \vspace{2mm}\\Hence, due to \eqref{eq_Girsanov_naive},
    invoking the inequality $(a+b)^2\leq 4(a^2+b^2)$, we find
    \begin{align}
        \textnormal{KL}(P\|\tilde{P}^\textnormal{ex}) 
        \leq & 
        \mb{E}\,\Big[
        \frac{2}{\sigma^2}
        \sum_{0\leq k\leq K-1} \int_{t_k}^{t_{k+1}}
        \|\hat{u} _{t_k}(Y_{t_k})
        -u_{t_k}(Y_{t_k})\|^2\,dt
        \Big]
        \\ \label{eq_second_term_KL_pathwise} & +
        \mb{E}\,\Big[
        \frac{2}{\sigma^2}
        \sum_{0\leq k\leq K-1} \int_{t_k}^{t_{k+1}}
        \|u_{t_k}(Y_{t_k})
        -u_t(Y_t)\|^2\,dt
        \Big].
    \end{align}
    By Assumption \ref{assumption_posterior_estimator_SL_process}, the first term is bounded by $\frac{2}{\sigma^2} t_K \varepsilon_0^2$. 
    Combining this with Fubini and the result of Lemma \ref{lemma_E_bound} for the second term, we obtain in total
    that
    \begin{align}
        \textnormal{KL}(P\|\tilde{P}^\textnormal{ex}) 
        \leq 
        \frac{2}{\sigma^2} t_K \varepsilon_0^2
        +2d\cdot C_{\textnormal{disc}}
        ,
    \end{align}
    which together with \eqref{eq_data_pinsker} gives the claim of Lemma \ref{lemma_KL_discretization_error}.
\end{proof}

\section{Optimality of log-SNR adapted discretization}\label{appendix_discretization}
In this section, we establish our result on the optimality of the log-SNR adapted discretization, e.g.\ Lemma \ref{lemma_opt_step_size}. Our approach is similar to the one in \cite{score_based_improved}. We also include a discussion at the end of this section related to our a priori imposed condition on $t_{K-1}$.
\begin{proof}[Proof of Lemma \ref{lemma_opt_step_size}]
    Recall that since we are interested in how to choose $t_1<\dots<t_{K-1}$ given $t_0<t_K$, the initial and final time $t_0,t_K$ are fixed.
    Observe that, taking $(t_1,\dots,t_{K-1})\in [t_0,t_{K}]^{K-1}$, we trivially have
    \begin{align}
        f(t_1,\dots,t_{K-1})\coloneqq & \Big(\sum_{1\leq k\leq K-1} \big((t_{k+1}-t_k)-(t_{k}-t_{k-1})\big)
        \frac{1}{t_k}\Big)
        + \frac{t_1-t_0}{t_0}
        \\ \label{eq_optimality_lower_bound} \leq &
        \Big(\sum_{1\leq k\leq K-1} \max\big\{0,(t_{k+1}-t_k)-(t_{k}-t_{k-1})\big\}
        \frac{1}{t_k}\Big)
        + \frac{t_1-t_0}{t_0}.
    \end{align}
    Given $t_{K-1}\in (t_0,t_{K})$, we define the set of possible discretization choices between $t_0$ and $t_{K-1}$ as
    \begin{align}
        \Gamma_{t_{K-1}} = 
        \{(t_k)_{1\leq k \leq K-2}:\,t_0<t_1<...<t_{K-2}<t_{K-1}\}.
    \end{align}
    Further, we define
    \begin{align}
        f_{t_{K-1}}\in C(\overline{\Gamma_{t_{K-1}}},\mb{R})\cap C^\infty(\Gamma_{t_{K-1}},\mb{R})   
    \end{align}
    for $(t_1,\dots,t_{K-2})\in \overline{\Gamma_{t_{K-1}}}$ by
    \begin{align}
        f_{t_{K-1}}(t_1,\dots,t_{K-2})\coloneqq f(t_1,\dots,t_{K-2},t_{K-1}).
    \end{align}
    Due to \eqref{eq_optimality_lower_bound}, if we establish that the log-SNR adapted discretization between $t_0$ and 
    $t^\ast_{K-1}\coloneqq \big(\frac{t_K}{t_0}\big)^\frac{K-1}{K}\,t_0$, e.g.\ 
    \begin{align}
        (t_k)_{1\leq k \leq K-2}=\Big(\big(\frac{t_{K-1}^\ast}{t_0}\big)^\frac{k}{K-1}\, t_0\Big)_{1\leq k \leq K-2}
        =\Big(\big(\frac{t_K}{t_0}\big)^\frac{k}{K}\, t_0\Big)_{1\leq k \leq K-2},    
    \end{align}
    minimizes $f_{t^\ast_{K-1}}$, and that equality holds in \eqref{eq_optimality_lower_bound} when plugging in the minimizer, the result follows.
    
    We start by showing that for any $t_{K-1}\in (t_0,t_K)$ the log-SNR adapted discretization between $t_0$ and $t_{K-1}$ minimizes $f_{t_{K-1}}$. Computing $\partial_{t_k}f_{t_{K-1}}$ for $1\leq k \leq K-2$, we find
    \begin{align}\label{eq_partial_f}
        \partial_{t_k}f_{t_{K-1}}(t_1,\dots,t_{K-2})
        =& -\frac{t_{k+1}+t_{k-1}}{t_k^2}+\frac{1}{t_{k+1}}+\frac{1}{t_{k-1}}.
    \end{align}
    Basic algebra yields that setting the above to zero is equivalent to asking $\frac{t_{k+1}}{t_k} = \frac{t_k}{t_{k-1}}$.
    Hence, $\frac{t_{k+1}}{t_k}$ is constant across $1\leq k\leq K-2$, implying that the log-SNR adapted discretization between $t_0$ and $t_{K-1}$ is the unique minimization candidate in $\Gamma_{t_{K-1}}$. 
    To see that it is indeed the global minimizer on $\overline{\Gamma_{t_{K-1}}}$, we check that there are no minimizers on $\partial \Gamma_{t_{K-1}}$. Since $\overline{\Gamma_{t_{K-1}}}$ is compact and $f_{t_{K-1}}$ is continuous, the function $f_{t_{K-1}}$ must admit a minimizer on $\overline{\Gamma_{t_{K-1}}}$, which is then the log-SNR adapted discretization.
    \vspace{2mm}\\Let $(t_k)_{1\leq k \leq K-2}\in \partial \,\Gamma_{t_{K-1}}$. Then, there exists $1\leq k\leq K-2$ such that either $t_{k-1}<t_k=t_{k+1}$ or $t_{k-1}=t_k<t_{k+1}$. In the first case, the identity \eqref{eq_partial_f} becomes
    \begin{align}\notag
        \partial_{t_k}f_{t_{K-1}}(t_1,\dots,t_{K-1})
        =&-\frac{t_{k-1}}{t_k^2}+\frac{1}{t_{k-1}}.
    \end{align}
    Dividing by $t_{k-1}$, we find that this quantity is positive. Hence, in that case it is impossible that $(t_k)_{1\leq k \leq K-2}\in \partial \,\Gamma_{t_{K-1}}$ is a minimizer, as decreasing $t_k$ reduces the value of $f_{t_{K-1}}$. The argument for excluding the second case is analogous. Therefore, we have that $(t_k)_{1\leq k \leq K-2} = \big(  (\frac{t_{K-1}}{t_0})^{\frac{k}{K-1}}t_0\big)_{1\leq k \leq K-2}$ is the unique minimizer on $\overline{\Gamma_{t_{K-1}}}$.
    Since with the choice $t_{K-1}^\ast=\big(\frac{t_K}{t_0}\big)^{\frac{K-1}{K}}\,t_0$, for all $0\leq k\leq K-1$ we have $t_{k+1}-t_k=\big((\frac{t_{K}}{t_0})^{1/K}-1\big)\,t_k$, we also have that the step-sizes are non decreasing. Hence, when plugging the minimizer into \eqref{eq_optimality_lower_bound}, equality holds. This proves the Lemma.
\end{proof}
\begin{remark}
    Given the proof of Lemma \ref{lemma_opt_step_size}, a natural question is whether the minimizer of
    \begin{align}\notag
        [t_0,t_K]\ni t_{K-1}\mapsto
        \underset{(t_k)_{1\leq k\leq K-2}\in \Gamma_{t_{K-1}}}{\textnormal{min}}\,\,\,\,\,
            f_{t_{K-1}}(t_1,\dots,t_{K-2})
    \end{align}
    is given by the choice $t_{K-1}=\big(\frac{t_K}{t_0}\big)^{K-1/K}t_0$ that corresponds to the log-SNR discretization between $t_0$ and $t_K$. If this would be the case, we could prove the analogue of Lemma \ref{lemma_opt_step_size} when dropping the condition on $t_{K-1}$. This is \textit{not} the case.
    We deem this to be a technical issue in the following sense. We obtained the constant $C_{\textnormal{disc}}$ through the bound \eqref{eq_integral_minimization}, where we bounded 
    \begin{align}\notag
            &\sum_{1\leq k\leq K-1} \big((t_{k+1}-t_k)-(t_{k}-t_{k-1})\big)
            \mb{E}\,\big[
            \textnormal{Tr}\Cov (X|Y_{t_k})
            \big]
            \\ \notag & + (t_1-t_{0})
            \mb{E}\,\big[
            \textnormal{Tr}\Cov (X|Y_{t_0})
            \big]
            \underbrace{-(t_{K}-t_{K-1})\mb{E}\,\big[
            \textnormal{Tr}\Cov (X|Y_{t_K})
            \big]}_{(\ast)},
    \end{align}
    using $\mb{E}\,\big[\textnormal{Tr}\Cov (X|Y_{t})\big]\leq d\frac{\sigma^2}{t}$. There, we had to discard the last contribution $(\ast)$ in the above, as its sign is negative regardless of the choice of discretization, so the just stated inequality cannot be used. \textit{If} we could use it, which of course we cannot, we would get an additional $-d\sigma^2 \frac{t_K-t_{K-1}}{t_K}$ contribution in the bound \eqref{eq_integral_minimization}. Then, the associated $f^\ast$ that we would have to minimize would be $f^\ast\coloneqq f-\frac{t_K-t_{K-1}}{t_K}$ and the minimizer $(t_k)_{1\leq k \leq K-1}$ of this would be the log-SNR discretization between $t_0$ and $t_K$, without imposing any condition on $t_{K-1}$ a priori.
\end{remark}
\section{Theoretical results on the general observation process}\label{appendix_SDE_theory}
We discuss SDE descriptions for the standard and the general observation process, which are at the heart of SLIPS. 
In particular, we explain why the general observation process does not solve the SDE \eqref{eq_SDE_general} if $\alpha$ is non-linear, which is not discussed in \cite{SLIPS}. 
To do so, we start by deriving a path-dependent SDE for the general observation process, see Theorem \ref{Theorem_SDE_characterization} and Lemma \ref{lemma_SDE_characterization}. 
For the standard observation process, we show that the at first sight path-dependent SDE actually depends only on the current state, which is the content of Lemma \ref{lemma_posterior_exp_simplification} and Corollary \ref{corollary_standard_Y_SDE}. For the general observation process, this is not the case, see Remark \ref{remark_denoiser_simplifies}. Though, results of Brunick and Shreve \cite{brunick_shreve} imply that the marginals of the general observation process coincide with the marginals of an SDE, whose coefficients depend only on the current state, see Theorem \ref{theorem_SDE_characterization_endpoint}.
\vspace{2mm}\\We start with a general representation result 
that can be found in \cite{Liptser2001} as Theorem 7.12. 
\begin{theorem}[\cite{Liptser2001}, Theorem 7.12]\label{Theorem_SDE_characterization}
    Let $T,\sigma\in (0,\infty)$ be fixed and consider a perfectly filtered probability space $(\Omega,\mc{F},(\mc{F}_t)_{0\leq t <T},\mb{P})$. On that space, consider a Brownian motion $(B_t)_{t \in [0,T)}$ and an adapted continuous stochastic process $(\beta_t)_{t\in [0,T)}$ fulfilling 
    \begin{align}\label{eq_condition_liptser}
        \int_0^T \mb{E}\big[\|\beta_t\|\big]\,dt<\infty.
    \end{align}
    Assume that an adapted process $(Y_t)_{t\in [0,T)}$ is a weak solution to 
    \begin{align}\label{eq_start_SDE_general}
        \begin{cases}
            dY_t &= \beta_t\,dt+\sigma \,dB_t,
            \\ Y_0 &=0.
        \end{cases}
    \end{align}
    Further, assume we are given a measurable map \mbox{$h:[0,T)\times C([0,T),\mb{R}^d)\to \mb{R}^d$} such that almost surely for almost all $t\in [0,T)$
    we have $h_t (Y^{t}) = \mb{E}[\beta_t|\mc{F}^Y_t]$,
    where $Y^{t}$ denotes the stopped process $(Y_{s\wedge t})_{s\in [0,T)}$.
    Define the stochastic process $(\bar{B}_t)_{t\in [0,T)}$ by
    \begin{align}
        \bar{B}_t \coloneqq \frac{1}{\sigma} \Big( Y_t -\int_0^t h_s (Y^{s}) ds\Big).
    \end{align}
    Then, $\bar{B}$ is a Brownian Motion on $(\Omega,\mc{F}^Y,(\mc{F}^Y_t)_{0\leq t <T},\mb{P})$ and the process $(Y_t)_{t\in [0,T)}$ solves the path-dependent SDE
    \begin{align}
        \begin{cases}
            dY_t &= h_t (Y^{t})\, dt + \sigma d\bar{B}_t,
            \\ Y_0 &=0.
        \end{cases}
    \end{align}
\end{theorem}
\begin{remark}
    Observe that in the above setup, since we start with $\beta$ and $B$, a weak solution of \eqref{eq_start_SDE_general} is trivially given by $Y_t\coloneqq \int_0^t\,\beta_t\,dt+\sigma\,B_t$.
\end{remark}
\begin{proof}
    First, note that the condition \eqref{eq_condition_liptser} implies that $\mb{E}[\|\beta_t\|]<\infty$ for almost all $t\in [0,T)$.
    In particular, $\mb{E}[\beta_t|\mc{F}^Y_t]$ is well defined for almost all $t\in [0,T)$.
    The existence of a functional $h$ as in the Theorem is ensured by Lemma 4.9 from \cite{Liptser2001}.
     Observe that once it is established that $\bar{B}$ is a BM with respect to $(\mc{F}^Y_t)_{0\leq t <T}$, the SDE characterization of $Y$ follows directly. Indeed, we have
    \begin{align}\notag
        dY_t 
        = h_t (Y^t) dt + \sigma d\frac{1}{\sigma} \Big( Y_t -\int_0^t h_s (Y^{s}) ds\Big)
        =h_t (Y^{t})\,dt + \sigma\, d\bar{B}_t.
    \end{align}
    We now show that $\bar{B}$ is an $\mc{F}^Y$-BM using the L\'evy characterization of Brownian motion, see Theorem 3.16 in \cite{karatzas1991brownian}. The L\'evy characterization tells us that if $\bar{B}$ is a continuous local martingale with respect to $\mc{F}^Y$, it is a Brownian motion with respect to $\mc{F}^Y$ if and only if $[ \bar{B}^k,\bar{B}^l]_t = \delta_{k=l}\,t$.
    Observe that a priori we already know that $\bar{B}$ is adapted to $\mc{F}^Y$, since $\bar{B}_t$ is a function of $Y^t$, and that it is continuous by definition.    
    \vspace{2mm}\\We start demonstrating that
    $[\bar{B}^k,\bar{B}^l]_t = \delta_{k=l}\,t$. Observe that, since $t\mapsto \beta_t$ is continuous and since, as a consequence of $\eqref{eq_condition_liptser}$, the process $t\mapsto h_t(Y^t)$ is integrable a.s.\ on $[0,T)$, the processes $t\mapsto \int_0^t\, \beta_s\,ds$ and $t\mapsto \int_0^t\, h_s(Y^s)\,ds$ have finite variation. Hence, by the pathwise characterization of the quadratic variation and since $Y_t= \int_0^t\beta_s\,ds+\sigma B_t$, for $1\leq k,l\leq d$ we have
    \begin{align}
        & [\bar{B}^k,\bar{B}^l]_t = 
        \frac{1}{\sigma^2} [Y^k 
        ,Y^l]_t
        =
        \frac{1}{\sigma^2} [\sigma B^k
        ,
        \sigma B^l]_t
        =\delta_{k=l}\,t.
    \end{align}
    We now check that $\bar{B}$ is an $\mc{F}^Y$ local martingale. Observe first that the condition in \eqref{eq_condition_liptser} ensures that $\bar{B}_t\in L^1$ and that all involved conditional expectations are well defined, so we show directly that $\bar{B}$ is a true martingale without using stopping times. 
    To check the martingale property, we compute using Fubini for $0\leq s \leq t$ 
    \begin{align}\label{eq_computations_SDE_char}
        &\mb{E} [\bar{B}_t-\bar{B}_s|\mc{F}^Y_s] 
        = 
        \mb{E} [\int_s^t\, \beta_r \,dr+\sigma (B_t-B_s)-\int_s^t\, h_r(Y_r)\,dr
        |\mc{F}^Y_s]
        \\\notag = &
        \sigma \mb{E} [B_t-B_s
        |\mc{F}^Y_s]
        +
        \int_s^t\,\mb{E} [ \beta_r - h_r(Y_r)
        |\mc{F}^Y_s]\,dr
        \\\notag = &
        \sigma \mb{E}\big[\mb{E} [B_t-B_s
        |\mc{F}_s]\big|\mc{F}^Y_s\big]
        +
        \int_s^t\,\mb{E} [ \beta_r - \mb{E}[\beta_r|\mc{F}^Y_r]
        |\mc{F}^Y_s]\,dr
        \\ \notag =&0.
    \end{align}
    In the last and second to last equality we used the tower property of conditional expectation, namely that for $\sigma$-algebras $\mc{F}\subset \mc{G}$ and $X\in L^1$ we have $\mb{E}[\mb{E}[X|\mc{F}]|\mc{G}]=\mb{E}[\mb{E}[X|\mc{G}]|\mc{F}]=\mb{E}[X|\mc{F}]$. 
    Since by the above $\bar{B}$ is an $\mc{F}^Y$-continuous (true) martingale and since we verified that $[\bar{B}^k,\bar{B}^l]_t = \delta_{k=l}\,t$, it is an $\mc{F}^Y$-Brownian motion. Thus, the Theorem follows.
\end{proof}
\begin{remark}\label{remark_sde_sigma_algebra}
    Observe that the argument of the previous proof works only when working with the sigma algebras $\mc{F}^Y_t = \sigma \big((Y_s))_{0\leq s \leq t}\big)$, since they define a filtration. If instead one would consider the sigma algebras $\mc{G}_t = \sigma (Y_t)$, the above argument breaks. Indeed, $\mc{G}_t$ is not a filtration since in general $\sigma (Y_s) \nsubseteq \sigma (Y_t)$ for $s< t$, so the tower property cannot be applied.
\end{remark}
Next, we apply the Theorem to the general observation process $Y^\alpha_t$, which yields a path-dependent SDE description. 
\begin{lemma}\label{lemma_SDE_characterization}
    Given $T_{\textnormal{Gen}}>0$ and a denoising schedule $\alpha:[0,T_{\textnormal{Gen}})\to \mb{R}_+$, consider a perfectly filtered probability space $(\Omega,\mc{F},(\mc{F}_t)_{t\in [0,T_\textnormal{Gen})},\mb{P})$ on which the general observation process 
    $(Y_t^\alpha)_{t\in [0,T_{\textnormal{Gen}})}=(\alpha(t)X+\sigma B_t)_{t\in [0,T_{\textnormal{Gen}})}$ is defined, where $X\sim \pi$, $X\in \mc{F}_0$ and $B$ is a BM with respect to $(\mc{F}_t)_{t\in [0,T_\textnormal{Gen})}$.
    Assume $\pi \in \mc{P}_1(\mb{R}^d)$.     
    \vspace{2mm}\\Let $u^{\alpha,p}_{t}:[0,T_\textnormal{Gen})\times C([0,T_\textnormal{Gen}),\mb{R}^d)\to \mb{R}^d$ be such that almost surely for almost all $t\in [0,T_{\textnormal{Gen}})$
    \begin{align}\label{eq_functional_pathwise_general}
        u^{\alpha,p}_{t}(Y^{\alpha,t}) = \mb{E}[X|\mc{F}^{Y^\alpha}_t],
    \end{align}
    where $Y^{\alpha,t}$ denotes the stopped process $(Y^\alpha_{s\wedge t})_{s\in [0,T)}$.
    Then, the observation process $Y^\alpha_t$ solves on $[0,T_{\textnormal{Gen}})$ the SDE
    \begin{align}\label{eq_Y_SDE_lemma_general}
        \begin{cases}
            dY_t^\alpha &= \alpha' (t) u^{\alpha,p}_t (Y^{\alpha,t}) dt + \sigma d\bar{B}_t
            \\ Y^\alpha_0 &=0,
        \end{cases}
    \end{align}
    where $(\bar{B}_t)_{t\in [0,T_{\textnormal{Gen}})}$ is a BM with respect to $(\mc{F}^{Y^\alpha}_t)_{t\in [0,T_\textnormal{Gen})}$.
\end{lemma}
\begin{remark}
    Observe that the BM in \eqref{eq_Y_SDE_lemma_general} is \textit{not} the same as the BM in the definition of $Y_t^\alpha = \alpha (t) X + \sigma B_t$, but the one coming from Theorem \ref{Theorem_SDE_characterization}.
\end{remark}
\begin{proof}
The existence of a functional $u_t^{\alpha,p}$ as in \eqref{eq_functional_pathwise_general} is ensured again by Lemma 4.9 from \cite{Liptser2001}.
Note that $Y^\alpha_t= \alpha (t) X + \sigma B_t$ solves on $[0,T_{\textnormal{gen}})$ trivially the SDE
\begin{align}\label{eq_canonical_SDE_observation_process}
    dY^\alpha_t =  \alpha'(t)X\, dt +\sigma \, dB_t, 
\end{align}
since $\int^t_0 \alpha'(s) X ds= (\alpha(t)-\alpha(0))X =\alpha(t) X$. Hence, we want to apply Theorem \ref{Theorem_SDE_characterization} with $\beta_t = \alpha'(t) X$, which is continuous and adapted to $\mc{F}$, since $X\in \mc{F}_0$.
We now check that condition \eqref{eq_condition_liptser} holds when taking $T<T_{\textnormal{gen}}$. Observe that by Fubini
\begin{align}
    \mb{E}\Big[ \int^{T}_0 \|\alpha'(t) X\|\, dt\Big] =
    \int^{T}_0 \alpha'(t) \mb{E}\big[\| X\|\big]\,dt
    =\alpha(T)\mb{E}\big[\| X\|\big]<\infty,
\end{align}
where we used that $\alpha'$ is positive, since $\alpha$ is increasing, that $\alpha(0)=0$, and that $\mb{E}\big[\| X\|\big]<\infty$ by assumption. Hence, condition \eqref{eq_condition_liptser} indeed holds for $T<T_{\textnormal{gen}}$.
Thus, Theorem \ref{Theorem_SDE_characterization} yields that $(\bar{B}_t)_{t\in [0,T)}=\Big(\frac{1}{\sigma}\big(Y^\alpha_t-\int_0^t u_s^{\alpha,p}(Y^{\alpha,s})\,ds\big)\Big)_{t\in [0,T)}$ is a BM on $(\Omega,\mc{F}^{Y^\alpha}_T,(\mc{F}^{Y^\alpha}_t)_{t\in [0,T)},\mb{P})$ and on $[0,T)$
it holds that
\begin{align}
    dY^\alpha_t = \alpha'(t) u^{\alpha,p}_t (Y^{\alpha,t})dt + \sigma d\bar{B}_t.
\end{align}
Since $T< T_\textnormal{gen}$ was arbitrary, the result follows.
\end{proof}
Even though the drift coefficient in the
SDE \eqref{eq_Y_SDE_lemma_general} of Lemma \ref{lemma_SDE_characterization} involves the conditional expectation $\mb{E}[X|\mc{F}^{Y^\alpha}_t]$, in the case of linear $\alpha$ this conditional expectation coincides with the one where one conditions only on the endpoint. This is the content of the next Lemma, where we consider purely for notational simplicity only the standard case $\alpha(t)=t$.
\begin{lemma}\label{lemma_posterior_exp_simplification}
    Consider the standard SL observation process $(Y_t)_{t\geq 0} = (t X+\sigma B_t)_{t\geq 0}$ and assume $\pi \in \mc{P}_1(\mb{R}^d)$. For $t\geq 0$,
    $u_t(Y_t)$ is a version of the conditional expectation $\mb{E}[X|\mc{F}^Y_t]$.
    In particular, a.s.\ it holds that $\mb{E}[X|\mc{F}^Y_t]=\mb{E}[X|Y_t]$. As a consequence, $(u_t(Y_t))_{t\geq 0}$ is an $\mc{F}^Y$-martingale.
\end{lemma}
\begin{proof}
    We prove the first part of the result by showing that  
    the conditional expectation of $X$ with respect to any finite number of marginals of $Y^t$
    coincides with $\mb{E}[X|Y_t]$, e.g.\ we show 
    that for any $0<t_1<t_2<...<t_N =t$, we have a.s.
    \begin{align}\label{eq_posterior_exp_dependence_endpoint}
        \mb{E}[X|Y_t] = \mb{E}[X|(Y_{t_k})_{1\leq k \leq N}].
    \end{align}
    Considering for $N^\ast\in \mb{N}$ the discretization $(t^{N^\ast}_k)_{1\leq k\leq 2^{N^\ast}}\coloneqq(\frac{k}{2^{N^\ast}}t)_{1\leq k\leq 2^{N^\ast}}$, we have
    \begin{align}
        \mc{F}^Y_t = \sigma \Big(\underset{{N^\ast}\in \mb{N}}{\bigcup} 
        \,\,\,
        \sigma\big((Y_{t^{N^\ast}_k})_{1\leq k \leq 2^{N^\ast}}\big)
        \Big).
    \end{align}
    Provided \eqref{eq_posterior_exp_dependence_endpoint} holds, the result follows by an application of L\'evy's upward Theorem, as $\sigma\big((Y_{t^{N^\ast}_k})_{1\leq k \leq 2^{N^\ast}}\big)$ is an increasing family of $\sigma$-algebras, see Theorem 4.6.8 in \cite{durrett_probability}.
    \vspace{2mm}\\We now prove \eqref{eq_posterior_exp_dependence_endpoint}. As for the case of $\mb{E}[X|Y_t]$, by Bayes, given $y_{t_1},\dots,y_{t_N}\in \mb{R}^d$, the posterior distribution 
    \begin{align}
        q_{t_1<t_2...<t_N}(\cdot|(y_{t_k})_{1\leq k \leq N})\coloneqq \textnormal{Law}\big(X|(Y_{t_k})_{1\leq k \leq N}=(y_{t_k})_{1\leq k \leq N}\big)
    \end{align}
    can be defined for $x\in \mb{R}^d$ and $y_{t_1},\dots,y_{t_N}\in \mb{R}^d$
    up to the normalization constant through its density by 
    \begin{align}\label{eq_posterior_prop_multi}
        q_{t_1<t_2<...<t_N}(x|(y_{t_k})_{1\leq k \leq N})\propto_{y_{t_1},...,y_{t_N}} p_{t_1<t_2<...<t_N}((y_{t_k})_{1\leq k \leq N}|x)\pi (x),
    \end{align}
    where $p_{t_1<t_2<...<t_N}(\cdot|x)\coloneqq\textnormal{Law}(Y_{t_1},Y_{t_2},\dots,Y_{t_N}|X=x)$.
    Analogously to \eqref{eq_posterior_expectation}, the associated optimal denoiser is then given by
    \begin{align}\label{eq_denoiser_multi_input}
        u_{t_1<t_2<...<t_N} ((y_{t_k})_{1\leq k \leq N})\coloneqq \int_{\mb{R}^d} x \, 
        q_{t_1<t_2<...<t_N}(dx|(y_{t_k})_{1\leq k \leq N})
    \end{align}
    and $u_{t_1<t_2<...<t_N} \big((Y_{t_k})_{1\leq k \leq N}\big)$ is a version of the conditional expectation $\mb{E}[X|(Y_{t_k})_{1\leq k \leq N}]$.
    Thus, due to \eqref{eq_posterior_prop_multi} and \eqref{eq_denoiser_multi_input}, the property \eqref{eq_posterior_exp_dependence_endpoint} follows if we can show that 
    \begin{align}\label{eq_endpoint_density_prop_path}
        p((y_{t_k})_{1\leq k \leq N}|x) \propto
        p(y_{t_N}|x),
    \end{align}
    where the normalization constant hidden by $\propto$ is allowed to depend on anything except for $x$.
    Indeed, since in \eqref{eq_denoiser_multi_input} we integrate against $x$, factors not depending on $x$ cancel out in the normalization. Thus, this implies $u_{t_1<t_2<...<t_N}((y_{t_k})_{1\leq k \leq N}) = u_{t_N}(y_{t_N})$. We now verify \eqref{eq_endpoint_density_prop_path}. Conditional on $X$, the marginals 
    \begin{align}\notag
        Y_{t_{k+1}}-Y_{t_k}= (t_{k+1}-t_k)X+\sigma(B_{t_{k+1}}-B_{t_k})    
    \end{align}
    are Gaussian, independent across $1\leq k \leq N-1$ and independent of $Y_{t_1}=t_1X+\sigma B_{t_1}$. Hence, setting $t_{0}=0$ and $y_{0}=0$, we have 
    \begin{align}\notag
        p(y_{t_1},...,y_{t_N}|x) &= \prod_{0\leq k \leq N-1}
        \mc{N}\big(y_{t_{k+1}}-y_{t_{k}},(t_{k+1}-t_k)\cdot x,
        (t_{k+1}-t_k) \sigma^2 \cdot \textnormal{Id} \big).
    \end{align} 
    With this representation \eqref{eq_endpoint_density_prop_path} follows. Indeed, we thus have
    \begin{align}\notag
        p(y_{t_1},...,y_{t_N}|x)
        &\propto_{t_1,\dots,t_N} 
        \prod_{0\leq k \leq N-1}
        \exp \Big(-\frac{\|\big(y_{t_{k+1}}-y_{t_{k}}\big)-(t_{k+1}-t_k)\cdot x\|^2}{2(t_{k+1}-t_k)\sigma^2}\Big)
        \\ \notag 
        &=
        \exp \Big(-\sum_{0\leq k\leq N-1}\frac{\|y_{t_{k+1}}-y_{t_{k}}\|^2}{2(t_{k+1}-t_k)\sigma^2}\Big)
        \\ \notag &\,\,\,\,
        \prod_{0\leq k \leq N-1}
        \exp \Big(\frac{2\langle y_{t_{k+1}}-y_{t_{k}},x\rangle-(t_{k+1}-t_k)\cdot \|x\|^2}{2\sigma^2}\Big)
        \\ \notag 
        &= 
        \exp \Big(-\sum_{0\leq k\leq N-1}\frac{\|y_{t_{k+1}}-y_{t_{k}}\|^2}{2(t_{k+1}-t_k)\sigma^2}\Big)
        \exp \Big(\frac{2\langle y_{t_N},x\rangle-t_{N}\cdot \|x\|^2}{2\sigma^2}\Big)
        \\ \notag &\propto_{y_{t_0},\dots,y_{t_N}} 
        \exp \Big(-\frac{\| y_{t_N}-t_N x\|^2}{2t_N \sigma^2}\Big)
        \propto_{t_N} 
        p_{t_N}(y_{t_N}|x)=p_{t}(y_{t}|x).
    \end{align}
    Since in the above, all constants hidden by $\propto$ depend only on $t_k$ and $y_{t_k}$ but not on $x$, \eqref{eq_endpoint_density_prop_path} follows and thus the claim of the Lemma. The martingale property follows from the fact that $\mb{E}[X|\mc{F}_t^Y]$ is an $\mc{F}^Y$-martingale, as $(\mc{F}_t^Y)_{t\geq 0}$ is a filtration and $X\in L^1$.
\end{proof}
\begin{remark}\label{remark_denoiser_simplifies}
    Observe that for the case of a general denoising schedule $\alpha$, the argument of Lemma \ref{lemma_posterior_exp_simplification} does not work. While one can define the posterior measures analogously and use the same strategy as in the standard case, the equation \eqref{eq_endpoint_density_prop_path} does not hold in general.
    Indeed, in that case, we do not have the telescope property 
    \begin{align}
        \sum_{0\leq k \leq N-1}\frac{\langle y_{t_{k+1}}-y_{t_{k}},(\alpha (t_{k+1})-\alpha (t_{k}))x\rangle}{t_{k+1}-t_k} 
        = 
        \langle y_{t_{N}},x\rangle,
    \end{align}
    unless $\alpha$ is linear. 
    Therefore, $\mb{E}[X|Y^\alpha_t]$ is not a version of the conditional expectation $\mb{E}[X|\mc{F}^{Y^\alpha}_t]$.
    Instead, when taking the limit of step sizes to zero, a stochastic integral depending on the whole path history arises.  Independent of us, the fact that $\mb{E}[X|\mc{F}^{Y^\alpha}_t]$ does not depend only on $Y_t^\alpha$ but on the whole path history has been pointed out as well in the recent work \cite{SL_joint_applications_2025}, Remark 2.1 therein.
\end{remark}
As a direct consequence of Lemma \ref{lemma_posterior_exp_simplification}, the SDE description for the standard observation process can be simplified.
\begin{corollary}\label{corollary_standard_Y_SDE}
    In the setup of Lemma \ref{lemma_SDE_characterization},
    in the standard case $\alpha(t) = t$ and $T_{\textnormal{Gen}}=\infty$, the
    standard observation process $Y$ solves
    \begin{align}\label{eq_Y_SDE_lemma}
        \begin{cases}
            dY_t &= u_t (Y_{t}) dt + \sigma d\bar{B}_t
            \\ Y_0 &=0,
        \end{cases}
    \end{align}
    where $\bar{B}$ is the BM from Lemma \ref{lemma_SDE_characterization}.
\end{corollary}
\begin{proof}
    The Corollary is a direct consequence of Lemma \ref{lemma_posterior_exp_simplification}. Indeed,
    since $u_t(Y_t)$ defines a version of $\mb{E}[X|\mc{F}^Y_t]$ by Lemma \ref{lemma_posterior_exp_simplification}, we can define for $\omega \in C([0,\infty),\mb{R}^d)$ the functional $h_t$ of Lemma \ref{lemma_SDE_characterization} as $h_t(\omega)\coloneqq u_t(\omega_t)$, so that $h_t(Y^t)= u_t(Y_t)$.
\end{proof}
For the general observation process, the SDE of Lemma \ref{lemma_SDE_characterization} does not simplify, as $\mb{E}[X|Y^\alpha_t]\neq \mb{E}[X|\mc{F}^{Y^\alpha}_t]$. 
Brunick and Shreve though showed in \cite{brunick_shreve}, Corollary 3.7, that if one replaces in the canonical SDE \eqref{eq_canonical_SDE_observation_process} of $Y_t^\alpha$ the drift $X$ by its conditional expectation given the current state $Y_t^\alpha$, the SDE has a solution with marginals coinciding with the one of $Y_t^\alpha$, provided $\pi \in \mc{P}_1(\mb{R}^d)$. We state the precise result of \cite{brunick_shreve} adapted to our context next, see also Proposition 10 and Corollary 11 in the most recent arXiv version of \cite{SLIPS}.
\begin{theorem}[\cite{brunick_shreve}, Corollary 3.7]\label{theorem_SDE_characterization_endpoint}
    Given $T_{\textnormal{Gen}}>0$ and a denoising 
    schedule \mbox{$\alpha:[0,T_{\textnormal{Gen}})\to \mb{R}_+$}, consider the associated general observation process $(Y_t^\alpha)_{t\in [0,T_{\textnormal{Gen}})}=(\alpha(t)X+\sigma B_t)_{t\in [0,T_{\textnormal{Gen}})}$.
    Denote by $u_t^\alpha:\mb{R}^d\to \mb{R}^d$ a function such that a.s.\ $u_t^\alpha (Y_t^\alpha)=\mb{E}[X|Y_t^\alpha]$, defined analogously to \eqref{eq_posterior_expectation}.
    Assume
    $\pi \in \mc{P}_1(\mb{R}^d)$.
    \vspace{2mm}\\Then, there exists a pair of continuous processes $(\widehat{Y}_t^\alpha,\bar{B}_t)_{0\leq t < T_{\textnormal{gen}}}$ defined on a filtered probability space
    $(\Omega,\mc{F},(\mc{F}_t)_{0\leq t<T_{\textnormal{gen}}},\mb{P})$, where $(\bar{B}_t)_{0\leq t <T_{\textnormal{gen}}}$ is a Brownian motion on that space, that solve for $t\in [0,T_{\textnormal{Gen}})$ the SDE
    \begin{align}\label{eq_Y_SDE_general_markovian}
        \begin{cases}
            d\widehat{Y}_t^\alpha &= \alpha' (t) u^\alpha_t (\widetilde{Y}^\alpha_{t}) dt + \sigma d\bar{B}_t,
            \\ \widetilde{Y}^\alpha_0 &=0.
        \end{cases}
    \end{align}
    Further, 
    $\widehat{Y}_t^\alpha$ fulfills for all $0\leq t<T_{\textnormal{gen}}$
    \begin{align}\label{eq_marginals_coincide}
        \textnormal{Law} (\widehat{Y}_t^\alpha)= \textnormal{Law}(Y_t^\alpha)=p_t^\alpha.
    \end{align}
\end{theorem}
Note that existence of solutions to \eqref{eq_Y_SDE_general_markovian} can be established under mild linear growth conditions on the drift coefficient of \eqref{eq_Y_SDE_general_markovian} via Girsanov's Theorem, see Proposition 3.6 in \cite{karatzas1991brownian}. The main contribution of the above Theorem is thus that the marginal laws coincide. While the proof of this result is highly technical, the key idea is to show that the Markovian dynamics of \eqref{eq_Y_SDE_general_markovian} change the marginal distribution infinitesimally in the same way as the dynamics $\alpha'(t)X\,dt+\sigma\,dB_t$, which is achieved by a discretization argument.
\end{document}